\documentclass[a4paper,11pt]{article}       
\usepackage{amsmath}
\usepackage{amssymb}
\usepackage{mathptmx}      
\usepackage[utf8]{inputenc}
\usepackage[T1]{fontenc}
\usepackage[english]{babel}
\usepackage{booktabs}
\usepackage{multirow}
\usepackage{hyperref}
\usepackage{setspace}
\usepackage{url}
\usepackage[labelfont=bf, font=small, textfont=normalfont]{caption}
\usepackage{graphicx,amsmath,amssymb}
\usepackage{multirow}
\usepackage{tabu}
\usepackage{makecell}
\usepackage{bbding}
\usepackage{xspace}							
\usepackage{dsfont}
\usepackage{bbm}							
\usepackage[algo2e, linesnumbered, ruled, vlined]{algorithm2e}

\begin{document}

\newtheorem{theorem}{Theorem}								
\newtheorem{lemma}[theorem]{Lemma}							
\newtheorem{corollary}[theorem]{Corollary}              





%
\newcommand*{\oneone}{$(1{+}1)$~\textup{EA}\xspace}
\newcommand*{\OM}{\textrm{\textup{\textsc{OneMax}}}\xspace}
\newcommand*{\nwspace}{\hspace*{.1em}}
\newcommand*{\Or}{\mathrm O}
\newcommand*{\Om}{\mathrm \Omega}
\newcommand*{\xorg}{x_{\mathrm{orig}}}
\newcommand*{\dominates}{\succcurlyeq}
\newcommand*{\sdominates}{\succ}
\newcommand*{\mutate}{\texttt{mutate}}
\newcommand*{\cross}{\texttt{cross}}
\newcommand*{\uar}{u.a.r.\xspace}
\newcommand*{\wrt}{w.r.t.\xspace}
\newcommand*{\Prob}{\textup{Pr}}
\newcommand*{\loss}{\textup{loss}}
\newcommand*{\surplus}{\textup{surplus}}
\newcommand{\qed}{\hspace*{\fill}\blackslug}
\newenvironment{proof}{$\;$\newline \noindent {\sc Proof.}$\;\;\;$\rm}{\qed}
\newcommand{\proofend}{\quad\blackslug}
\newcommand{\blackslug}{\penalty 1000\hbox{
    \vrule height 8pt width .4pt\hskip -.4pt
    \vbox{\hrule width 8pt height .4pt\vskip -.4pt
          \vskip 8pt
      \vskip -.4pt\hrule width 8pt height .4pt}
    \hskip -3.9pt
    \vrule height 8pt width .4pt}}


\newenvironment{myAlgorithm}%
	{\vspace{.75em}
	\begin{center}
	\begin{minipage}{.75\linewidth}
	\begin{algorithm2e}[H]
	\setstretch{1.25}}%
	{\end{algorithm2e}
	\end{minipage}
	\end{center}
	\vspace*{.75em}
  }

\title{\bf Runtime Performances of Randomized Search Heuristics\\ for the Dynamic Weighted Vertex Cover Problem
\thanks{A preliminary version of this work was presented at the {\it 2018 Genetic and Evolutionary Computation Conference} (GECCO)~\cite{shi2018GECCO}. The work is supported by the National Natural Science Foundation of China under Grants 61802441, 61672536, 61836016, and the Australian Research Council (ARC) through Grant DP160102401. 
}}
\author{
 \vspace*{3mm}
 {\sc Feng Shi} $^{\mbox{\footnotesize \textdagger}}$ \ \ \ \
 {\sc Frank Neumann}$^{\mbox{\footnotesize \textdaggerdbl}}$ \ \
 {\sc Jianxin Wang}$^{\mbox{\footnotesize \textdagger}}$\\
   $^{\mbox{\footnotesize \textdagger}}$School of Computer Science and Engineering, Central South University,  \\ 
   Changsha 410083, P.R. China, fengshi@csu.edu.cn\\
   $^{\mbox{\footnotesize \textdaggerdbl}}$School of Computer Science, The University of Adelaide, Adelaide, Australia  
   }

\date{}
\maketitle

\vspace{-7mm}

\begin{abstract}
	Randomized search heuristics such as evolutionary algorithms are frequently applied to dynamic combinatorial optimization problems. Within this paper, we present a dynamic model of the classic Weighted Vertex Cover problem and analyze the runtime performances of the well-studied algorithms Randomized Local Search and (1+1) EA adapted to it, to contribute to the theoretical understanding of evolutionary computing for problems with dynamic changes. In our investigations, we use an edge-based representation based on the dual form of the Linear Programming formulation for the problem and study the expected runtime that the adapted algorithms require to maintain a 2-approximate solution when the given weighted graph is modified by an edge-editing or weight-editing operation. Considering the weights on the vertices may be exponentially large with respect to the size of the graph, the step size adaption strategy is incorporated, with or without the 1/5-th rule that is employed to control the increasing/decreasing rate of the step size. Our results show that three of the four algorithms presented in the paper can recompute 2-approximate solutions for the studied dynamic changes in polynomial expected runtime, but the (1+1) EA with 1/5-th Rule requires pseudo-polynomial expected runtime.
\end{abstract}

\section{Introduction}
\label{sec:intro}
Over the past decades, randomized search heuristics such as evolutionary algorithms and ant colony optimization have been applied successfully in various areas, including engineering and economics. To gain a deep insight into the behaviors of evolutionary algorithms, many theoretical techniques for analyzing their expected runtime were presented~\cite{Auger11,ncs/Jansen13,BookNeuWit}. And using these techniques, evolutionary algorithms designed for some classic combinatorial optimization problems have been studied. In particular, the Vertex Cover problem plays a crucial role in the area~\cite{friedrich2010approximating,hansen2003reducing,jansen2013approximating,kratsch2013fixed,pourhassan2016ppsn}.

Consider an instance $I$ of a given combinatorial optimization problem, and a solution $S$ to $I$ satisfying a specific quality guarantee (optimal or approximated). If an operation on $I$ results in a new instance $I'$, which is similar to $I$
(the similarity between the two instances depends on the scale of the operation),
then a natural and interesting problem arises: Is it easy to find a solution $S'$ to $I'$ that satisfies the specific quality guarantee, starting from the original solution $S$? 
In other words, how much runtime does a specific algorithm take to get a solution $S'$ to $I'$ with the quality guarantee,  starting from $S$? 
The above setting is referred as the {\it dynamic model} of the given combinatorial optimization problem.

Studying the performances of evolutionary algorithms for dynamic models of combinatorial optimization problems is an emerging field in evolutionary computation
\cite{friedrich2017s,kotzing20151+,neumann2015runtime,pourhassan2015maintaining,roostapour2018performance,shi2019reoptimization}. 
Within the paper, we present the dynamic model of the Weighted Vertex Cover problem (WVC), which is simply named Dynamic Weighted Vertex Cover problem (DWVC). Our goal is to analyze the behaviors of the well-studied algorithms Randomized Local Search (RLS) and (1+1) EA that are adapted to DWVC. More specifically, we study the expected runtime (i.e., the expected number of fitness evaluations) that the algorithms need to recompute a 2-approximate solution when the given weighted graph is modified by a graph-editing operation, starting from a given 2-approximate solution to the original weighted graph.
Note that all weighted graphs considered in the paper are vertex-weighted, i.e., the weight function is defined on the vertices, not edges.

{\bf Related work.} For the Vertex Cover problem, it is well-known that under the Unique Games Conjecture~\cite{khot2002power}, there does not exist an approximation algorithm with a constant ratio $r < 2$, unless P = NP~\cite{khot2008vertex}.
The best-known 2-approximation algorithm for the Vertex Cover problem is based on the maximal matching: Construct a maximal matching by greedily adding edges, then let the vertex cover contain both endpoints of each edge in the maximal matching.
For WVC, Hochbaum~\cite{LPForWeighteVC1983} 
showed that a 2-approximate solution can be obtained by using the Linear Programming (LP) result of the fractional WVC. 
Du et al.~\cite{du2011design} found that a maximal solution to the dual form~\cite{vazirani2013approximation} of the LP formulation (simply called dual formulation) for the fractional WVC also directly induces a 2-approximate solution. 
Using this conclusion, Bar-Yehuda and Even~\cite{bar1981linear} presented a linear-time 2-approximation algorithm for WVC. The essential difference between the primal form of the LP formulation (simply called primal formulation) and the dual formulation for the fractional WVC is: The primal formulation considers the problem from the perspective of vertices; the dual formulation considers it from the perspective of edges~\cite{du2011design}. (More details of the LP formulation and its dual formulation for the fractional WVC can be found in the next section.)

Pourhassan et al.~\cite{pourhassan2015maintaining} presented a dynamic model of the Vertex Cover problem, in which the graph-editing operation adds (resp., removes) exactly one edge into (resp., from) the given unweighted graph, and analyzed evolutionary algorithms with respect to their abilities to maintain a 2-approximate solution. 
They examined different variants of the RLS and (1+1) EA, node-based representation and edge-based representation.
If using the node-based representation, they gave classes of instances for which both algorithms cannot get a 2-approximate solution in polynomial expected runtime with high probability.
However, using the edge-based representation, they showed that the RLS and (1+1) EA can maintain 2-approximations efficiently if the algorithms start with a search point corresponding to a maximal matching of the original unweighted graph and use the fitness function given in~\cite{jansen2013approximating}  penalizing the edges sharing vertices. 

Inspired by the work of Pourhassan et al.~\cite{pourhassan2015maintaining} and the essential difference between the primal and dual formulations of the fractional WVC, it is promising to consider DWVC from the perspective of edges, i.e., utilize the dual formulation to analyze DWVC.
Here we give another example to show that using the dual formulation is better than using the primal formulation, to analyze DWVC. 
Consider a simplest graph-editing operation that removes or adds exactly one edge $[v,v']$.
For the primal formulation, if a new edge $[v,v']$ is added into the graph, 
then the corresponding LP values of $v$ and $v'$ may be required to increase as their sum may be $< 1$ with respect to the given original LP solution; if an edge $[v,v']$ is removed from the graph, then the corresponding LP values of $v$ and $v'$ may have the room to decrease with respect to the given original LP solution.
Thus there are two possible adjustment directions for the LP values of the vertices if using the primal formulation. 
For the dual formulation, because the given original maximal solution to the dual formulation does not violate the corresponding LP constraints no matter whether the edge $[v,v']$ is removed or added, so we only need to consider increasing the LP values of the edges. 
Therefore, using the dual formulation is able to simplify the analysis for DWVC, compared to using the primal formulation.

We formulate DWVC in the paper as: Given a weighted graph $G = (V,E,W)$ and a maximal solution $S$ to the dual formulation of the fractional WVC on $G$, the goal is to find a maximal solution to the dual formulation of the fractional WVC on the weighted graph $G^* = (V^*,E^*,W^*)$ starting from the original maximal solution $S$, where $G^*$ is obtained by one of the two graph-editing operations on $G$:
(1) replace the edge-set $E$ with a new one $E^*$; (2) replace the weight function $W$ with a new one $W^*$.
The version of DWVC with edge-modification is denoted by DWVC-E, and the one with weight-modification is denoted by DWVC-W. 
Denote by $D \in \mathds{N}^+$ the scale of the graph-editing operation, more specifically, $D = |(E^* \setminus E) \cup (E \setminus E^*)|$ or $D = |\{v \in V|W(v) \neq W^*(v)\}|$.  
It is necessary to point that both $G$ and $G^*$ are simple graphs (at most one edge between any two vertices).

Recently Pourhassan et al.~\cite{pourhassan2017use} studied WVC using the dual formulation of the fractional WVC. As the weights on the vertices may be exponentially large with respect to the size of the graph (the number of edges), they incorporated the {\it Step Size Adaption} strategy~\cite{beyer2002evolution} into their (1+1) EA (Algorithm 4 given in~\cite{pourhassan2017use}). However, their (1+1) EA was shown to take exponential expected runtime with high probability to get a maximal solution to the dual formulation. 
There are two factors causing the long runtime of their algorithm. Firstly, for a mutation $M$ constructed by their (1+1) EA, there may exist two edges selected by $M$ whose LP values are increased and decreased respectively. 
The randomness on the adjustment direction of the LP values leads that a mutation increases the sum of LP values for the edges with a relatively small probability, i.e., a mutation is rejected with a relatively large probability. 
Secondly, for a mutation $M$ that is rejected by their (1+1) EA, the step sizes of all the edges selected by $M$ would be decreased. Under the combined impact of the two factors, the step sizes of the edges cannot be increased enough to overcome the exponentially large weights on the vertices. That is, the step size adaption strategy is nearly invalid for their (1+1) EA. 

{\bf Contributions}. Drawing on the experience of work~\cite{pourhassan2017use} due to Pourhassan et al., we give two algorithms (1+1) EA and RLS adapted to DWVC with the step size adaption strategy as well. To avoid the invalidation of the step size adaption strategy that happens in the algorithm of~\cite{pourhassan2017use}, the two algorithms adopt an extra policy with three points: (1) the LP values of the edges selected by a mutation either all increase or all decrease (this only applies to the (1+1) EA, because any mutation of the RLS selects exactly one edge); (2) whether the algorithms increase or decrease the LP values of the edges depends on the fitness of the maintained solution; and (3) the condition to decrease the step size of a specific edge is very strict. 
Under the cooperation of the step size adaption strategy and the policy given above, the (1+1) EA and RLS are shown to take expected runtime $\Or \big(\alpha m \log_{\alpha} W_{\textup{max}} \cdot \log (\max\{\alpha m, \alpha D \cdot W_{\textup{max}} \}) \! \big)$ to solve the two versions of DWVC (including two special variants for DWVC-E, and two special variants for DWVC-W), where $m$ denotes the number of edges in $G^*$, $W_{\textup{max}} \ge 1$ denotes the maximum weight that the vertices in $G$ and $G^*$ have, and $\alpha \in \mathds{N}^+ \setminus \{1\}$ denotes the increasing/decreasing rate of the step size (i.e., the increment on the LP value for each edge can be exponentially increased or decreased by a factor $\alpha$).

\begin{table*}[t]
\vspace*{.25cm}
\scriptsize
\begin{center}
\renewcommand{\arraystretch}{1}
\begin{tabular}{@{}lcccc@{}}
\toprule
 & {\bf RLS} \ or \ \textbf{(1+1) EA} & {\bf \mbox{\textup{RLS with 1/5-th Rule}}} & {\bf \mbox{\textup{(1+1) EA with 1/5-th Rule}}} &\\
\midrule
{\bf DWVC-E$^+$} &
  $\Or \big(\alpha m \log_{\alpha} W_{\textup{max}} \cdot \min\{D,\log (\alpha D \cdot \log_{\alpha} W_{\textup{max}})\} \! \big)$ &
  $\Or \big(\alpha m D \log_{\alpha} W_{\textup{max}} \cdot \log W_{\textup{max}} \! \big)$ &
  $\Om (2^{m^{\epsilon/2}})$, $0 < \epsilon \leq 1/2$ & \\
{\bf DWVC-E$^-$} &
  $\Or \big(\alpha m \log_{\alpha} W_{\textup{max}} \cdot \log (\max\{\alpha m, \alpha D \cdot W_{\textup{max}} \}) \! \big)$ &
  $\Or \big(\alpha m \log_{\alpha} W_{\textup{max}} \cdot \min\{m \log W_{\textup{max}}, D \cdot W_{\textup{max}} \}  \! \big)$ &
  $\Om (2^{m^{\epsilon/2}})$, $0 < \epsilon \leq 1/2$ \\
{\bf DWVC-E} &
  $\Or \big(\alpha m \log_{\alpha} W_{\textup{max}} \cdot \log (\max\{\alpha m, \alpha D \cdot W_{\textup{max}} \}) \! \big)$ &
  $\Or \big(\alpha m \log_{\alpha} W_{\textup{max}} \cdot \min\{m \log W_{\textup{max}}, D \cdot W_{\textup{max}} \}  \! \big)$ &
  $\Om (2^{m^{\epsilon/2}})$, $0 < \epsilon \leq 1/2$ \\ 
\cmidrule{1-5}
{\bf DWVC-W$^+$} &
  $\Or \big(\alpha m \log_{\alpha} W_{\textup{max}} \cdot \log (\max\{\alpha m, \alpha D \cdot W_{\textup{max}} \}) \! \big)$ &
  $\Or \big(\alpha m \log_{\alpha} W_{\textup{max}} \cdot \min\{m \log W_{\textup{max}}, D \cdot W_{\textup{max}} \}  \! \big)$ &
  $\Om (2^{m^{\epsilon/2}})$, $0 < \epsilon \leq 1/2$ \\ 
{\bf DWVC-W$^-$} &
  $\Or \big(\alpha m \log_{\alpha} W_{\textup{max}} \cdot \log (\max\{\alpha m, \alpha D \cdot W_{\textup{max}} \}) \! \big)$ &
  $\Or \big(\alpha m \log_{\alpha} W_{\textup{max}} \cdot \min\{m \log W_{\textup{max}}, D \cdot W_{\textup{max}} \}  \! \big)$ &
  $\Om (2^{m^{\epsilon/2}})$, $0 < \epsilon \leq 1/2$ \\ 
{\bf DWVC-W} &
  $\Or \big(\alpha m \log_{\alpha} W_{\textup{max}} \cdot \log (\max\{\alpha m, \alpha D \cdot W_{\textup{max}} \}) \! \big)$ &
  $\Or \big(\alpha m \log_{\alpha} W_{\textup{max}} \cdot \min\{m \log W_{\textup{max}}, D \cdot W_{\textup{max}} \}  \! \big)$ &
  $\Om (2^{m^{\epsilon/2}})$, $0 < \epsilon \leq 1/2$ \\ 

\bottomrule\\
\end{tabular}
\end{center}
\vspace*{-5mm}
\caption{Overview on runtime performances of the four algorithms for the two versions of DWVC, DWVC-E and DWVC-W, including the two special variants for DWVC-E (DWVC-E$^+$ and DWVC-E$^-$) and the two special variants for DWVC-W (DWVC-W$^+$ and DWVC-W$^-$).
The notation $m$ denotes the number of edges in the new graph, $W_{\textup{max}}$ denotes the maximum weight that the vertices in the original and new graphs have, $D$ denotes the scale of the graph-editing operation, and $\alpha \in \mathds{N}^+ \setminus \{1\}$ denotes the increasing/decreasing rate of the step size that is an integer ranging from 2 to $W_{\textup{max}}$.
The lower bound of the runtime of the \mbox{\textup{(1+1) EA with 1/5-th Rule}} for DWVC holds with probability $1- e^{-{\rm \Omega}(m^{\epsilon})}$ if $W_{\textup{max}} \ge \alpha^{m}$.
}
\label{table:overviewResults}
\end{table*}

For the extra policy given above, its last two points play an important role in avoiding the invalidity of the step size adaption strategy, but they seem too restrictive and a little artificial. Thus we introduce the 1/5-th (success) rule, and give two algorithms with both the 1/5-th rule and step size adaption strategy, called the (1+1) EA with 1/5-th Rule and RLS with 1/5-th Rule. 
The 1/5-th rule is one of the best-known techniques in parameter control, especially in the control of the mutation probability (for more details, please refer to~\cite{doerr2015optimal}). In the paper, we use the 1/5-th rule to control the increasing/decreasing rate of the step size. 
More specifically, the LP values of the edges selected by a mutation of the two algorithms with the 1/5-th rule increase or decrease with the same probability 1/2 (i.e., not depend on the maintained solution). If the mutation is accepted, then the step sizes of the selected edges are increased by a factor $\alpha$; otherwise, decreased by a factor $\alpha^{1/4}$.
For the RLS with 1/5-th Rule, we show that it can solve the two versions of DWVC (including the four special variants) efficiently. However, for the (1+1) EA with 1/5-th Rule, we construct a special instance for each version of DWVC, and show that the algorithm takes at least pseudo-polynomial time to solve it. 
The main results given in the paper are summarized in Table~\ref{table:overviewResults}.

The rest of the paper is structured as follows. We start by giving the related definitions and problem formulations in Section 2. 
Then we present the algorithms (1+1) EA and RLS (with the step size adaption strategy), and algorithms \mbox{\textup{(1+1) EA with 1/5-th Rule}} and \mbox{\textup{RLS with 1/5-th Rule}} for DWVC in two separated subsections of Section 3. 
For the two versions of DWVC, Sections 4 and 5, respectively, analyze the expected runtime of the (1+1) EA and RLS, and the \mbox{\textup{(1+1) EA with 1/5-th Rule}} and \mbox{\textup{RLS with 1/5-th Rule}}.
Finally, conclusions are presented in Section 6.

\section{Preliminaries}
\label{sec:prelims}

Consider a weighted graph $G=(V,E,W)$ with a vertex-set $V=\{v_1,\ldots, v_n\}$, an edge-set $E=\{e_1,\ldots, e_m\}$, and a weight function $W: V \rightarrow \mathds{N}^+$ on the vertices. For any vertex $v \in V$, denote by $N_G(v)$ the set containing all the neighbors of $v$ in $G$, and by $E_G(v)$ the set containing all the edges incident to $v$ in $G$. For any vertex-subset $V' \subseteq V$, let $E_G(V') = \bigcup_{v \in V'} E_G(v)$. For any edge $e \in E$, denote by $E_G(e)$ the set containing all the edges in $G$ that have a common endpoint with $e$. For any edge-subset $E' \subseteq E$, let $E_G(E') = \bigcup_{e \in E'} E_G(e) \setminus E'$.

A vertex-subset $V_c \subseteq V$ is a {\it vertex cover} of $G$ if for each edge $e \in E$, where $e$ can be represented by its two endpoints $v$ and $v'$ as $[v,v']$, at least one of its two endpoints $v$ and $v'$ is in $V_c$. The weight of $V_c$ is defined as the sum of the weights on the vertices in $V_c$, written $\sum_{v \in V_c} W(v)$. 
The Weighted Vertex Cover problem (WVC) on the weighted graph $G$ asks for a vertex cover of $G$ with the minimum weight,  among all vertex covers of $G$.

Using the node-based representation (i.e. the search space is $\{0,1\}^n$, and for any solution $x = x_1 \ldots x_n$ the node $v_i$ is chosen iff $x_i=1$), the Integer Linear Programming (ILP) formulation for WVC is given as follows.
\begin{eqnarray*}
&& min \ \sum_{i=1}^n W(v_i)\cdot x_i \\
st. && x_i+x_j\geq 1 \ \ \ \ \ \forall \ [v_i,v_j]\in E \\
&& x_i\in \{0,1\} \ \ \ \ \ \ i = 1, ..., n
\end{eqnarray*}

By relaxing the constraint $x_i\in \{0,1\}$ of the ILP given above to $x_i\in [0,1]$, the Linear Programming (LP) formulation for the fractional WVC is obtained. Hochbaum~\cite{LPForWeighteVC1983} showed that a 2-approximate solution can be found by using the LP result of the fractional WVC --- include all the vertices $v_i$ with $x_i\geq 1/2$.
The dual form of the LP formulation (or simply called dual formulation) for the fractional WVC is given as follows, where $Y: E \rightarrow \mathds{R}^+ \cup \{0\}$ denotes a value assignment on the edges.
\begin{eqnarray*}
&& max \ \sum_{e \in E} Y(e)\\
st. && \sum_{e \in E_G(v)} Y(e) \leq W(v) \ \ \ \ \  \forall \ v\in V
\end{eqnarray*}

The value assignment $Y$ is called a {\it dual-solution} of $G$ in the paper. Given a vertex $v \in V$, it {\it satisfies} the {\it dual-LP constraint} with respect to the dual-solution $Y$ if $\sum_{e \in E_G(v)} Y(e) \leq W(v)$. Similarly, for an edge $e \in E$, it {\it satisfies} the dual-LP constraint with respect to $Y$ if both its endpoints satisfy the dual-LP constraint with respect to $Y$.
The dual-solution $Y$ of $G$ is {\it feasible} if all the vertices in $G$ satisfy the dual-LP constraint with respect to $Y$; otherwise, {\it infeasible}. The vertex $v \in V$ is {\it tight} with respect to $Y$ if $\sum_{e \in E_G(v)} Y(e) = W(v)$, and the edge $e \in E$ is {\it tight} with respect to $Y$ if at least one of its two endpoints is tight with respect to $Y$.
Given a dual-solution $Y$ of $G$, denote by $V_G(Y)$ the set containing all the vertices in $G$ that do not satisfy the dual-LP constraint with respect to $Y$, and by $E_G(Y)$ the set containing all the edges that are incident to the vertices in $V_G(Y)$.

A {\it maximal feasible dual-solution} (MFDS) of $G$ is a feasible dual-solution of $G$ such that none of the edges can be assigned a larger LP value without violating the dual-LP constraint.
Given an MFDS $Y$ of $G$, it induces a vertex cover of $G$ with ratio 2 directly, which contains all tight vertices with respect to $Y$ (a formal proof about the approximate ratio can be found in Theorem 8.4 of~\cite{du2011design}).

Two versions of the Dynamic Weighted Vertex Cover problem (DWVC) are studied in the paper, whose formal formulations are given below.

\begin{quote}
DWVC with Edge Modification ({\bf DWVC-E}) \\
{\it Input}: \ \ \ a weighted graph $G = (V,E,W)$, an MFDS $Y_{\textup{orig}}$ of $G$, and a new edge-set $E^*$; \\
{\it Output}: an MFDS $Y^*$ of $G^* = (V,E^*, W)$

\medskip

DWVC with Weight Modification ({\bf DWVC-W}) \\
{\it Input}: \ \ \ a weighted graph $G = (V,E,W)$, an MFDS $Y_{\textup{orig}}$ of $G$, and a new weight function $W^*$; \\
{\it Output}: an MFDS $Y^*$ of $G^* = (V,E, W^*)$
\end{quote}

There are two special variants for DWVC-E, DWVC-E$^+$ and DWVC-E$^-$. 
Let $E^+ = E^* \setminus E$ and $E^- = E \setminus E^*$. 
The variants DWVC-E$^+$ and DWVC-E$^-$ consider the cases $E^- = \emptyset$ and $E^+ = \emptyset$, respectively.
Similarly, there are two special variants for DWVC-W, DWVC-W$^+$ and DWVC-W$^-$. 
Let $V^+ = \{v \in V|W^*(v) > W(v) \}$ and $V^- = \{v \in V|W^*(v) < W(v) \}$.
The variants DWVC-W$^+$ and DWVC-W$^-$ consider the cases $V^- = \emptyset$ and $V^+ = \emptyset$, respectively.

\section{Four Adaptive Algorithms}
\label{sec:algorithms}
We start with the subsection that introduces the (1+1)~EA and RLS, then the subsection that introduces the \mbox{\textup{(1+1)~EA with 1/5-th Rule}} and \mbox{\textup{RLS with 1/5-th Rule}}, finally the subsection that gives two simple lemmata based on the selection mechanism of the four algorithms. 

It is worthy to point out that the solution maintained by the four algorithms is actually a $m$-dimensional vector (recall that $m$ is the number of edges in $G^* = (V^*,E^*,W^*)$), $[x_1,x_2,\ldots,x_m]$, in which each $x_i$ ($1 \le i \le m$) is the LP value of the edge $e_i \in E^*$. 
We always use notation $M$ to denote a mutation of the four algorithms, which corresponds to an adjustment (increment or decrement) on the values of some elements (corresponding to the edges chosen by $M$) in the vector.

\subsection{(1+1)~EA and RLS}
Consider two weighted graphs $G = (V,E,W)$ and $G^* = (V^*,E^*,W^*)$, where $G^*$ is obtained by one of the two graph-editing operations mentioned above on $G$. We study the expected runtime (i.e., the expected number of fitness evaluations) of the (1+1) EA and RLS, given in Algorithm~\ref{alg:(1+1) EA} and~\ref{alg:RLS} respectively, to find an MFDS of $G^*$ starting with a given MFDS $Y_{\textup{orig}}$ of $G$ (not from scratch).
The two algorithms run in a similar way, except the mechanism selecting edges for mutation. The (1+1) EA selects each edge in $E^*$ with probability $1/m$ at each iteration, resulting in an edge-subset $I$ containing all the selected edges (see step 7 of Algorithm~\ref{alg:(1+1) EA}), and adjusts the LP values of the edges in $I$. The RLS differs from the (1+1)~EA by selecting exactly one edge in $E^*$ in each round.

The two algorithms share the same general idea: If $Y_{\textup{orig}}$ is also a feasible dual-solution of $G^*$, then they directly increase the LP values of the edges in $G^*$ until the LP value of any edge cannot be assigned with a larger value under the dual-LP constraint (i.e, an MFDS of $G^*$ is found if the claimed condition is met). 
Note that no infeasible dual-solution would be accepted during the process. 
If $Y_{\textup{orig}}$ is an infeasible dual-solution of $G^*$, then the two algorithms first decrease the LP values of the edges in $E_{G^*}(Y_{\textup{orig}})$ (because only the vertices in $V_{G^*}(Y_{\textup{orig}})$ violate the dual-LP constraint with respect to $Y_{\textup{orig}}$), aiming to get a feasible dual-solution $Y_t$ of $G^*$ as soon as possible, afterwards, increase the LP values of the edges in $G^*$ to get an MFDS based on $Y_t$. 

The general idea of the two algorithms shows that the feasibility of the maintained solution decides the adjustment directions of the LP values of the selected edges. Thus we give a sign function $s(Y)$ below, to judge whether or not the considered solution $Y$ is a feasible dual-solution of $G^*$. 
\begin{displaymath}
  s(Y) = \left\{
     \begin{array}{lr}
       -1 & \ \ $if$ \ V_{G^*}(Y) \neq \emptyset, \ \ \ $i.e.,$ \ Y \ $is$ \ $infeasible$\\
       1 & $feasible$\\
     \end{array}
   \right.
\end{displaymath}
It is necessary to point out that a mutation of the (1+1) EA may choose more than one edge, and the LP values of the chosen edges are required to be either all increased or all decreased. 
In addition to the sign function $s()$, we also present a function $f(Y',Y)$ to compare the fitness of $Y'$ and $Y$, where $Y'$ is the dual-solution obtained by a mutation $M$ on the dual-solution $Y$ maintained by the two algorithms. It is defined as follows: $f(Y',Y) \ge 0$ if $Y'$ is not worse than $Y$; $f(Y',Y) < 0$ otherwise. 
\begin{equation*}
f(Y',Y) = 
\begin{cases} 
s(Y') \cdot \sum_{e \in E^*} \big(Y'(e) - Y(e) \! \big) & \text{if $s(Y) = 1$} \\
\sum_{e \in E_{G^*}(Y)} \big(Y(e) - Y'(e) \! \big) - m \cdot W_{\textup{max}} \cdot \sum_{e \in E^* \setminus E_{G^*}(Y)} \big |Y(e) - Y'(e) \! \big| & \text{if $s(Y) = -1$}
\end{cases}
\end{equation*}

By the general idea of the two algorithms given above, if $Y$ is a feasible dual-solution of $G^*$, then the two algorithms increase the LP values of the edges, thus we always have that $\sum_{e \in E^*} Y'(e) \ge \sum_{e \in E^*} Y(e)$ for the obtained offspring $Y'$. 
If $Y'$ is infeasible, then $s(Y') = -1$ and $f(Y',Y) < 0$; otherwise, $s(Y') = 1$ and $f(Y',Y) \ge 0$.
If $Y$ is an infeasible dual-solution of $G^*$, then the two algorithms
decrease the LP values of the edges firstly, aiming to get a feasible dual-solution of $G^*$. 
Note that the LP values of the edges in $E^* \setminus E_{G^*}(Y)$ do not need to be decreased as they satisfy the dual-LP constraint with respect to $Y$.
If they are decreased during the process to get the first feasible dual-solution, then the algorithm may spend much extra time to make up the decrements on the LP values of the edges in $E^* \setminus E_{G^*}(Y)$ (i.e., spend extra time to make the edges in $E^* \setminus E_{G^*}(Y)$ be tight again). 
Thus the term of $f(Y',Y)$,  
$$-m \cdot W_{\textup{max}} \cdot \sum_{e \in E^* \setminus E_{G^*}(Y)} \big|Y(e) - Y'(e) \! \big| \ ,$$
penalizes the mutation that decreases the LP values of the edges in $E^* \setminus E_{G^*}(Y)$, which guides the mutation to decrease only the LP values of the edges in $E_{G^*}(Y)$.
More specifically, if the LP value of some edge in $E^* \setminus E_{G^*}(Y)$ is changed by the considered mutation (note that the increment or decrement on the LP value is always $\geq 1$), then 
\begin{eqnarray}
\label{eqn:fitness1}
-m \cdot W_{\textup{max}} \cdot \sum_{e \in E^* \setminus E_{G^*}(Y)} \big|Y(e) - Y'(e) \! \big| \leq - m \cdot W_{\textup{max}} \ .
\end{eqnarray} 

\begin{myAlgorithm}
  \caption{(1+1) EA}
   \label{alg:(1+1) EA}
    Initialize solution $Y$ and step size function $\sigma : E^* \rightarrow 1$ \; 
    \tcp{\small{$Y(e) = Y_{\textup{orig}}(e)$ for each $e \in E^* \cap E$, and $Y(e) = 0$ for each $e \in E^* \setminus E$}}
    Determine $s(Y)$ \;
    \While {the termination criteria not satisfied}
        {
        $Y' := Y$ and $I := \emptyset$  \tcp*{\small{set $I$ keeps all edges chosen by the mutation}}
        \For {each edge $e \in E^*$ with probability $1/m$}
            {    
                $Y'(e) := \max \{ Y(e) + s(Y) \cdot \sigma(e), 0\}$ \;
                $I := I \cup \{e\}$ \;
                    
            }
            Determine $s(Y')$ and $f(Y',Y)$ \;
        \eIf {$f(Y',Y) \ge 0$}
             {
                $Y := Y'$ \;
                $\sigma(e) := \min \{\alpha \cdot \sigma(e), \alpha^{\lceil \log_{\alpha} W_{\textup{max}} \rceil +1} \}$ for all $e \in I$ \; 
                \tcp{\small{$\alpha$ is the increasing/decreasing rate of the step size}}         
              }
             {
                \If {$s(Y) > 0$}
                {
                    Let $I'$ be the subset of $I$ such that each edge $e \in I'$ violates the dual-LP constraint with respect to $Y'$, but no other edge in $I$ shares the endpoint that violates the dual-LP constraint with $e$ \;
                    and $\sigma(e) := \max \{\sigma(e)/\alpha,1\}$ for all $e \in I'$ \;
                }
             }
        }
\end{myAlgorithm}

Now we consider the upper bound of the term $\sum_{e \in E_{G^*}(Y)} \big(Y(e) - Y'(e) \! \big)$ under the assumption that the LP value of some edge in $E^* \setminus E_{G^*}(Y)$ is decreased by the considered mutation.
Since all solutions obtained during the process from the initial solution $Y_{\textup{orig}}$ to $Y$ are infeasible (including $Y_{\textup{orig}}$), $Y(e) \le Y_{\textup{orig}}(e) \le W_{\textup{max}}$ for each edge $e \in E^*$, i.e., $0 \le Y(e) - Y'(e) \le W_{\textup{max}}$. 
Moreover, as $E^* \setminus E_{G^*}(Y)$ cannot be empty, $|E_{G^*}(Y)| < m$. Therefore,  
\begin{eqnarray}
\label{eqn:fitness2}
\sum_{e \in E_{G^*}(Y)} \big(Y(e) - Y'(e) \! \big) < m \cdot W_{\textup{max}} \ .
\end{eqnarray} 
Combining Inequalities~\ref{eqn:fitness1} and~\ref{eqn:fitness2}, $f(Y',Y) < 0$ no matter whether $Y'$ is feasible or infeasible, implying that the mutation would be rejected if it changes the LP value of some edge in $E^* \setminus E_{G^*}(Y)$.
For the case that no LP value of the edges in $E^* \setminus E_{G^*}(Y)$ is decreased by the considered mutation, it is easy to see that $f(Y',Y) \ge 0$.

\begin{myAlgorithm}
  \caption{RLS}
   \label{alg:RLS}

    Initialize solution $Y$ and step size function $\sigma : E^* \rightarrow 1$ \; 
    \tcp{\small{$Y(e) = Y_{\textup{orig}}(e)$ for each $e \in E^* \cap E$, and $Y(e) = 0$ for each $e \in E^* \setminus E$}}
    Determine $s(Y)$ \;
    \While {the termination criteria not satisfied}
        {
            $Y' := Y$ \;
            Choose an edge $e \in E^*$ uniformly at random \;
            $Y'(e) := \max \{Y(e) + \sigma(e) \cdot s(Y), 0\}$ \;
            Determine $s(Y')$ and $f(Y',Y)$ \;
        \eIf {$f(Y',Y) \ge 0$}
             {$Y := Y'$ and $\sigma(e) := \min \{\alpha \cdot \sigma(e), \alpha^{\lceil \log_{\alpha} W_{\textup{max}} \rceil +1} \}$ \;
             \tcp{\small{$\alpha$ is the increasing/decreasing rate of the step size}} 
             }
             {
                \If {$s(Y) > 0$}
                {
                    $\sigma(e) := \max \{\sigma(e)/\alpha,1\}$ \;
                }
             }
        }
\end{myAlgorithm}

To deal with the case that the weights on the vertices are exponentially large with respect to the size of the graph (the number $m$ of edges), the Step Size Adaption strategy~\cite{beyer2002evolution} is incorporated into the two algorithms
(see steps 9-15 of Algorithm~\ref{alg:(1+1) EA} and steps 8-12 of Algorithm~\ref{alg:RLS}): the increment (called {\it step size} in the following text) on the LP values of the edges can exponentially increase or decrease. Let $\sigma : E^* \rightarrow \mathds{N}^+$ be the step size function that keeps the step size for each edge in $E^*$, and let $\sigma$ be initialized as $\sigma : E^* \rightarrow 1$.  

Given a mutation of the RLS on $Y$, if it is accepted (i.e., $f(Y',Y) \geq 0$, where $Y'$ is the solution obtained by the mutation on $Y$), then the step size of the chosen edge $e$ is increased by a factor $\alpha$, where $\alpha$ is an integer between 2 and $W_{\textup{max}}$; otherwise, decreased by a factor $\alpha$ if $s(Y) > 0$.
W.l.o.g., we assume that the step size of each edge can be upper and lower bounded by $\alpha^{\lceil \log_{\alpha} W_{\textup{max}} \rceil +1}$ and 1, respectively.
Given a mutation of the (1+1) EA on $Y$ resulting $Y'$, if it is accepted, then the step size of each edge $e \in I$ is increased by a factor $\alpha$; otherwise, the step size of each edge $e \in I'$ is decreased by a factor $\alpha$ if $s(Y) > 0$, where $I'$ is the subset of $I$ such that each edge $e \in I'$ violates the dual-LP constraint with respect to $Y'$, but no other edge in $I$ shares the endpoint that violates the dual-LP constraint with $e$ (see step 14 of Algorithm~\ref{alg:(1+1) EA}). 
The reason why we define the subset $I'$ of $I$ is that we can ensure that the step size of each edge in $I'$ is unfit for $Y$. 
For an edge $e$ in $I \setminus I'$, there are two cases: (1) neither its two endpoints violates the dual-LP constraint with respect to the dual-solution $Y'$; (2) there is another edge $e' \in I \setminus \{e\}$ that has a common endpoint with $e$ such that the common endpoint of $e$ and $e'$ violates the dual-LP constraint with respect to the dual-solution $Y'$. For case (1), we should not decrease its step size. For case (2),
we cannot conclude that the step size of $e$ is unfit for the solution $Y$, because the step size of $e$ may be fit for $Y$ if it is considered independently. 
If the algorithms adopt a ``radical'' strategy that decreases the step sizes of all the edges in $I$ if the mutation is rejected, then they would spend much time on increasing the step sizes of the edges (in some extreme case, the step size cannot exponentially increase, resulting in an exponential waiting time to get an MFDS~\cite{pourhassan2017use}). Thus we adopt a ``conservative" strategy: Only decrease the step sizes of the edges in $I'$.

Note that for any mutation of the (1+1) EA or RLS that is rejected, the step sizes of the edges selected by the mutation are not decreased if $s(Y) < 0$, because the rejection of the mutation is caused by the selection of the edges, not the violation of the dual-LP constraint.

\subsection{(1+1)~EA with 1/5-th Rule and RLS with 1/5-th Rule}

To eliminate the artificial influences on the two algorithms given in the previous subsection, such as the adjustment direction (increasing or decreasing the LP values) controlled by the sign function $s()$, and the strict condition to decrease the step size of a specific edge given in the (1+1) EA (only the step sizes of the edges in $I'$ can be decreased if the mutation is rejected and $s() > 0$), we incorporate the 1/5-th (success) rule, and present two algorithms, the (1+1) EA with 1/5-th Rule and RLS with 1/5-th Rule, given in Algorithm~\ref{alg:(1+1) EA with 1/5-th Rule} and~\ref{alg:RLS with 1/5-th Rule} respectively.  
The two algorithms follow the fitness comparing function $f(Y',Y)$ defined in the previous subsection. 

\begin{myAlgorithm}
  \caption{(1+1) EA with 1/5-th Rule}
   \label{alg:(1+1) EA with 1/5-th Rule}

    Initialize solution $Y$ and step size function $\sigma : E^* \rightarrow 1$ \; 
    \tcp{\small{$Y(e) = Y_{\textup{orig}}(e)$ for each $e \in E^* \cap E$, and $Y(e) = 0$ for each $e \in E^* \setminus E$}}
    \While {the termination criteria not satisfied}
        {
        $Y' := Y$ and $I := \emptyset$ \tcp*{\small{set $I$ keeps all edges chosen by the mutation}}
        Choose $b \in \{-1,1\}$ uniformly at random \;
        \For {each edge $e \in E^*$ with probability $1/m$}
            {
                $Y'(e) := \max \{ Y(e) + b \cdot \sigma(e), 0\}$ \;
                $I := I \cup \{e\}$ \;
            }
            Determine $f(Y',Y)$ \;
        \eIf {$f(Y',Y) \ge 0$}
             {
                $Y := Y'$ \;
                $\sigma(e) := \min \{\alpha \cdot \sigma(e), \alpha^{\lceil \log_{\alpha} W_{\textup{max}} \rceil +1} \}$ for all $e \in I$ \;
                \tcp{\small{$\alpha$ is the increasing/decreasing rate of the step size}} 
              }
             {
                $\sigma(e) := \max \{\alpha^{-1/4} \cdot \sigma(e),1\}$ for all $e \in I$ \;
             }
        }
\end{myAlgorithm}

The general idea of the two algorithms is: no matter whether or not the current maintained dual-solution is feasible, they either increase or decrease the LP values of the edges selected by the mutation of the algorithms with the same probability 1/2 (depend on the value of $b$, see step 4 of Algorithm~\ref{alg:(1+1) EA with 1/5-th Rule} and step 5 of Algorithm~\ref{alg:RLS with 1/5-th Rule}). If the mutation is accepted, then the dual-solution is updated, and the step sizes of these chosen edges are increased by a factor $\alpha$; otherwise, the step sizes of these chosen edges are decreased by a factor $\alpha^{1/4}$. It is necessary to point out that for a mutation of the (1+1) EA with 1/5-th Rule, we still require that the LP values of the edges selected by the mutation either all increase or all decrease.

The previous subsection analyzed the cases of the fitness comparing function $f(Y',Y)$ when the LP values of the edges are increased if $s(Y) =1$, or the LP values of the edges are decreased if $s(Y) =-1$. 
Here we supplement the analysis of the case that the LP values of the edges are decreased if $s(Y) =1$ , and the case that the LP values of the edges are increased if $s(Y) =-1$. 
If $s(Y) =1$ and the LP values of the edges are decreased, then obviously $s(Y') =1$, and 
$$f(Y',Y) = \sum_{e \in E^*} \left(Y'(e) - Y(e) \! \right) < 0 \ .$$ 
If $s(Y) = -1$ and the LP values of some edges are increased, then  
$$f(Y',Y) = \sum_{e \in E_{G^*}(Y)} \left(Y(e) - Y'(e) \! \right) - m \cdot W_{\textup{max}} \cdot \sum_{e \in E^* \setminus E_{G^*}(Y)} \left|Y(e) - Y'(e) \! \right| < 0 \ .$$

\begin{myAlgorithm}
  \caption{RLS with 1/5-th Rule}
   \label{alg:RLS with 1/5-th Rule}

    Initialize solution $Y$ and step size function $\sigma : E^* \rightarrow 1$ \; 
    \tcp{\small{$Y(e) = Y_{\textup{orig}}(e)$ for each $e \in E^* \cap E$, and $Y(e) = 0$ for each $e \in E^* \setminus E$}}
    \While {the termination criteria not satisfied}
        {
            $Y' := Y$ \;
            Choose an edge $e \in E^*$ uniformly at random \;
            Choose $b \in \{-1,1\}$ uniformly at random \;
            $Y'(e) := \max \{Y(e) + b \cdot \sigma(e), 0\}$ \;
            Determine $f(Y',Y)$ \;
        \eIf {$f(Y',Y) \ge 0$}
             {$Y := Y'$ and $\sigma(e) := \min \{\alpha \cdot \sigma(e), \alpha^{\lceil \log_{\alpha} W_{\textup{max}} \rceil +1} \}$ \;
             \tcp{\small{$\alpha$ is the increasing/decreasing rate of the step size}} 
             }
             {
                    $\sigma(e) := \max \{\alpha^{-1/4} \cdot \sigma(e),1\}$ \;
             }
        }
\end{myAlgorithm}

\subsection{Observations based on Fitness Comparing Function}

The selection mechanism of the four algorithms given above implies the following two lemmata.

\begin{lemma}
\label{lem:sign function}
Given two dual-solutions $Y$ and $Y'$, where $Y'$ is obtained by a mutation of the four algorithms on $Y$, if $Y'$ is accepted then $s(Y') \geq s(Y)$.
\end{lemma}


\begin{lemma}
\label{lem:infeasible to MFDS}
Given two dual-solutions $Y$ and $Y'$, where $Y'$ is obtained by a mutation of the four algorithms on $Y$,
if $Y$ is infeasible and $Y'$ is accepted, then the mutation only decreases the LP values of the edges in $E_{G^*}(Y)$.
\end{lemma}

\section{Runtime Analysis for the (1+1)~EA and RLS}

We start the section with a notion related to mutation, which plays an important role in the following discussion.
Given an edge $e$ in the weighted graph $G^*$, a mutation of the (1+1) EA or RLS is a {\it valid mutation on $e$} if it results in an increment or decrement on the LP value of $e$, or on the step size $\sigma(e)$ of $e$. Note that if the mutation is of the (1+1) EA, then it may choose some other edges in addition to $e$.
The two lemmata given below study the behaviors of the (1+1) EA and RLS on a specific edge $e^* = [v_1,v_2]$ in $G^*$.

\begin{lemma}
\label{lem:(1+1) EA for increment of G1}
Consider a feasible dual-solution $Y^{\dagger}$ of $G^*$, and an initial value $\sigma_1$ of the step size of the edge $e^*$. For a feasible dual-solution $Y^{\ddagger}$ obtained by the \mbox{\textup{(1+1) EA}} (or \mbox{\textup{RLS}}) starting with $Y^{\dagger}$, where $Y^{\ddagger}(e^*) - Y^{\dagger}(e^*) \geq \sigma_1$, the algorithm takes expected runtime $\Or \! \left(\alpha m \log_{\alpha} \left(Y^{\ddagger}(e^*) - Y^{\dagger}(e^*) \! \right) \! \right)$ to increase the LP value of $e^*$ from $Y^{\dagger}(e^*)$ to $Y^{\ddagger}(e^*)$. 
\end{lemma}

\begin{proof}
We start with the analysis for the (1+1) EA. 
Since $Y^{\dagger}$ is a feasible dual-solution of $G^*$, by Lemma~\ref{lem:sign function}, the sign function $s()$ remains at 1 during the process from $Y^{\dagger}$ to $Y^{\ddagger}$, indicating that the LP value of $e^*$ is monotonically increased from $Y^{\dagger}(e^*)$ to $Y^{\ddagger}(e^*)$.
Let $Y$ be an arbitrary accepted solution of the (1+1) EA during the process, $M$ be a mutation of the (1+1) EA on $Y$, and $Y'$ be the offspring obtained by $M$ on $Y$. 
In the following discussion, we first analyze the impact of the mutation $M$ on the step size $\sigma(e^*)$ of $e^*$, where the notation $\sigma(e^*)$ here denotes the step size of $e^*$ before the generation of $M$.
Observe that $M$ cannot influence $\sigma(e^*)$ if $e^* \notin I$, where $I$ denotes the set containing all the edges selected by $M$ (see step 7 of Algorithm~\ref{alg:(1+1) EA}). Thus in the following discussion, we assume that $e^* \in I$.

Case (1). $\sigma(e^*) \leq Y^{\ddagger}(e^*) - Y(e^*)$.  
If $M$ is accepted by the (1+1) EA, then the step size of $e^*$ is increased from $\sigma(e^*)$ to $\alpha \cdot \sigma(e^*)$;
otherwise, the analysis on $M$ is divided into the two subcases given below.

Case (1.1). An endpoint $v_1$ of $e^*$ violates the dual-LP constraint with respect to $Y'$. Since $\sigma(e^*) \leq Y^{\ddagger}(e^*) - Y(e^*)$, the edge-subset $(E_{G^*}(v_1) \cap I) \setminus \{e^*\}$ cannot be empty, and the increments on the LP values of the edges in $E_{G^*}(v_1) \cap I$ results in the dual-LP constraint violation on $v_1$ with respect to $Y'$.
According to the definition of the edge-set $I'$ (see step 14 of Algorithm~\ref{alg:(1+1) EA}), we have that $e^* \notin I'$, and $M$ cannot influence $\sigma(e^*)$.

Case (1.2). No endpoint of $e^*$ violates the dual-LP constraint with respect to $Y'$. According to the definition of the edge-set $I'$, we also have that $e^* \notin I'$, and $M$ cannot influence $\sigma(e^*)$.

By the above analysis, any mutation of the (1+1) EA cannot cause an decrement on the step size of $e^*$ under Case (1).
If the mutation $M$ only selects the edge $e^*$, then it is a valid mutation on $e^*$, and can be accepted by the algorithm.
The (1+1) EA generates such a valid mutation on $e^*$ with probability $\Om (1 / m)$. Thus under Case (1), the algorithm takes expected runtime $\Or (m)$ to increase the LP value of edge $e^*$ from $Y(e^*)$ to $Y(e^*) + \sigma(e^*)$, and increase the step size of $e^*$ from $\sigma(e^*)$ to $\alpha \cdot \sigma(e^*)$.

Case (2). $\sigma(e^*) > Y^{\ddagger}(e^*) - Y(e^*)$. For the case, the mutation $M$ would be rejected by the (1+1) EA as $e^* \in I$.
The analysis on $M$ can be divided into the following two subcases.

Case (2.1). There is no edge in $I \setminus \{e^*\}$ sharing the endpoint of $e^*$ that violates the dual-LP constraint with respect to $Y'$.  For this subcase, $e^* \in I'$, and the step size of $e^*$ is decreased from $\sigma(e^*)$ to $\sigma(e^*) / \alpha$.

Case (2.2). There is an edge $e^*_1 \in I \setminus \{e^*\}$ sharing the endpoint of $e^*$ that violates the dual-LP constraint with respect to $Y'$. Because of the existence of $e^*_1$, $e^* \notin I'$ and $M$ does not influence the step size of $e^*$.

If the mutation $M$ only selects the edge $e^*$, then it is valid mutation on $e^*$, and belongs to Case (2.1).
The (1+1) EA generates such a valid mutation on $e^*$ with probability $\Om(1 / m)$. 
Thus under Case (2), the (1+1) EA takes expected runtime $\Or (m)$ to decrease the step size of $e^*$ from  $\sigma(e^*)$ to $\sigma(e^*)/ \alpha$.

Now we are ready to analyze the expected runtime of the (1+1) EA to increase the LP value of $e^*$ from $Y^{\dagger}(e^*)$ to $Y^{\ddagger}(e^*)$, using the above obtained results. Since $Y^{\ddagger}(e^*) - Y^{\dagger}(e^*) \geq \sigma_1$, the whole process can be divided into Phase (I) and Phase (II).
Phase (I) contains all steps of the algorithm until the step size of $e^*$ is decreased for the first time, i.e., the step size of $e^*$ can only increase during the phase. More specifically, the condition of Case (1) is always met with respect to the maintained solution $Y$ during Phase (I). 
Phase (II) follows Phase (I), during which the step size of $e^*$ may increase or decrease, but the general trend is decreasing. 
W.l.o.g., assume that the initial value $\sigma_1$ of the step size of $e^*$ is equal to $\alpha^p$, where $p \geq 0$ is an integer not less than 0.

{\bf Phase (I)}. 
Let $q$ be the integer such that 
$$\sum_{i=p}^{q} {\alpha}^i \leq Y^{\ddagger}(e^*) - Y^{\dagger}(e^*) \ \ \textrm{and} \ \ \sum_{i=p}^{q+1} {\alpha}^i > Y^{\ddagger}(e^*) - Y^{\dagger}(e^*) \ .$$
Now it is easy to see that the step size of $e^*$ can be increased from $\alpha^p$ to ${\alpha}^{q+1}$ during the phase. Thus the number of valid mutations on $e^*$ required during Phase (I) is $q-p+1$, where
$$q-p+1 = \left\lfloor \log_{\alpha} \left(\frac{\big( Y^{\ddagger}(e^*) - Y^{\dagger}(e^*) \! \big) \left(\alpha-1 \right)}{\alpha^p} +1 \right) \right\rfloor \ .$$
Combining the expected runtime of the algorithm to generate a valid mutation on $e^*$, 
Phase (I) takes expected runtime $\Or \! \left(m \log_{\alpha} \left(Y^{\ddagger}(e^*) - Y^{\dagger}(e^*) \! \right) \! \right)$ (because $p$ may be 0).

{\bf Phase (II)}. During the phase, the LP value of $e^*$ is increased from $Y^{\dagger}(e^*) + \sum_{i=p}^{q} {\alpha}^i$ to $Y^{\ddagger}(e^*)$, and the step size of $e^*$ is decreased from ${\alpha}^{q+1}$ to 1. Similar to the analysis for Phase (I), we analyze the number $T$ of valid mutations on $e^*$ during Phase (II).
However, to simplify the analysis, we separately consider the number $t_i$ of valid mutations on $e^*$ with step size $\alpha^i$ among the $T$ valid mutations on $e^*$, where $0 \leq i \leq q+1$ (since the step size of $e^*$ can increase or decrease during Phase (II), there may be more than one valid mutation on $e^*$ with step size $\alpha^i$). Obviously $T = \sum_{i=0}^{q+1} t_i$.

We start with the analysis for $t_{q+1}$. Since the valid mutation on $e^*$ with step size $\alpha^{q+1}$ cannot be accepted, the step size will be decreased to $\alpha^{q}$. However, if a valid mutation on $e^*$ with step size $\alpha^{q}$ is accepted, then the step size will be increased to $\alpha^{q+1}$ again. Thus $t_{q+1} \leq 1 + (\alpha - 1) = \alpha$, because there are at most $\alpha - 1$ valid mutations on $e^*$ with step size $\alpha^{q}$ among the $T$ valid mutations on $e^*$ that can be accepted by the algorithm.
Now we consider $t_{i}$ for any $1 \leq i \leq q$, under the assumption that the mutation on $e^*$ with step size $\alpha^{i+1}$ cannot be accepted.
Using the reasoning similar to that given above for the mutation on $e^*$ with step size $\alpha^{q+1}$, we can get that there are at most $\alpha$ valid mutations on $e^*$ with step size $\alpha^{i}$ that can be rejected among the $T$ valid mutations on $e^*$. 
Combining it with the observation that there are at most $\alpha - 1$ valid mutations on $e^*$ with step size $\alpha^{i}$ that can be accepted, we can derive that $t_i \leq 2 \alpha -1$.
Once the step size of $e^*$ is decreased to 1, then the LP value of $e^*$ is between $Y^{\ddagger}(e^*) - \alpha + 1$ and $Y^{\ddagger}(e^*)$. If the LP value of $e^*$ equals $Y^{\ddagger}(e^*)$, then Phase (II) is over, and $t_0 = 0$. If the LP value of $e^*$ is between $Y^{\ddagger}(e^*) - \alpha + 1$ and $Y^{\ddagger}(e^*) -1$, then $t_0 \leq \alpha - 1$. The above analysis gives 
$$T = \sum_{i=0}^{q+1} t_i \leq (2 \alpha -1) \cdot (q+1) \ .$$  
By the analysis for Case (1-2), Phase (II) takes expected runtime $\Or \! \left(\alpha m \log_{\alpha} \left(Y^{\ddagger} \left(e^* \right) - Y^{\dagger}\left(e^* \right) \! \right) \! \right)$.

Summarizing the above analysis for the two phases, there are at most $2 \alpha (q+1)$ valid mutations on $e^*$ during the process from $Y^{\dagger}$ to $Y^{\ddagger}$, for which the (1+1) EA takes expected runtime
$\Or \! \left(\alpha m \log_{\alpha} \left(Y^{\ddagger} \left(e^* \right) - Y^{\dagger} \left(e^* \right) \! \right) \! \right)$.
Since the RLS chooses exactly one edge in each iteration, any mutation of the RLS on $e^*$ is valid. Using the reasoning similar to that given above, we can get the same expected runtime for the RLS.
\end{proof}

Now we analyze the expected runtime of the two algorithms to make the edge $e^*$ satisfy the dual-LP constraint, if they start with an infeasible dual-solution with respect to which $e^*$ violates the dual-LP constraint.

\begin{lemma}
\label{lem:(1+1) EA for decrement of G1}
Consider an infeasible dual-solution $Y^{\dagger}$ of $G^*$, with respect to which the edge $e^*$ violates the dual-LP constraint. For the first feasible dual-solution $Y^{\ddagger}$ obtained by the \mbox{\textup{(1+1) EA}} (or \mbox{\textup{RLS}}) starting with $Y^{\dagger}$, the algorithm takes expected runtime 
$\Or \! \left(m \log_{\alpha} \left(Y^{\dagger}(e^*) - Y^{\ddagger}(e^*) \! \right) \! \right)$ to decrease the LP value of $e^*$ from $Y^{\dagger}(e^*)$ to $Y^{\ddagger}(e^*)$. 
\end{lemma}

\begin{proof}
We start with the analysis for the (1+1) EA. Since $Y^{\ddagger}$ is the first feasible dual-solution obtained by the (1+1) EA starting with $Y^{\dagger}$, the LP value of $e^*$ is monotonically decreased from $Y^{\dagger}(e^*)$ to $Y^{\ddagger}(e^*)$.

Assume that the step size of $e^*$ is initialized as $\alpha^p$, where $p \geq 0$ is an integer not less than 0. Observe that the step size of $e^*$ cannot decrease during the process from $Y^{\dagger}$ to $Y^{\ddagger}$ because the sign function  remains at $-1$. 
Hence if $Y^{\ddagger}(e^*) > 0$, then there exists an integer $q$ such that $\sum_{i=p}^{q} \alpha^i = Y^{\dagger}(e^*) - Y^{\ddagger}(e^*)$, and the step size of $e^*$ is increased from $\alpha^p$ to $\alpha^{q+1}$ during the process.
Consequently, the process contains $q-p+1$ valid mutations on $e^*$, where
$$q-p+1 = \log_{\alpha} \left(\frac{\big(Y^{\dagger}(e^*) - Y^{\ddagger}(e^*) \big) \cdot \left(\alpha -1 \right)}{\alpha^p} +1 \right) \ .$$
If $Y^{\ddagger}(e^*) = 0$, then there exists an integer $q$ such that $\sum_{i=p}^{q-1} \alpha^i < Y^{\dagger}(e^*) $, $\sum_{i=p}^{q} \alpha^i \ge Y^{\dagger}(e^*)$, and the step size of $e^*$ is increased from $\alpha^p$ to $\alpha^{q+1}$ during the process.
Similarly, the process contains $q-p+1$ valid mutations on $e^*$, where
$$q-p+1 = \left\lceil \log_{\alpha} \left(\frac{\left(\alpha -1 \right) \cdot Y^{\dagger}(e^*) }{\alpha^p} +1 \right) \right\rceil \ .$$

The mutation that only selects the edge $e^*$ is a valid mutation on $e^*$, which can be generated by the (1+1) EA with probability $\Om(1 / m)$. Thus the (1+1) EA takes expected runtime 
$\Or \big(m(q+1) \! \big) = \Or \! \left(m \log_{\alpha} \left(Y^{\dagger}(e^*) - Y^{\ddagger}(e^*)\! \right) \! \right)$ to get $Y^{\ddagger}$ (because $p$ may be 0). The above conclusions for the (1+1) EA also apply to the RLS. 
\end{proof}

\subsection{Analysis for DWVC with Edge Modification}

We start the subsection with the analysis of the algorithms (1+1) EA and RLS for the two special variants of DWVC-E, namely, DWVC-E$^+$ and DWVC-E$^-$. 
Denote by $E^+ = E^* \setminus E$ the set containing all the new added edges, and by $E^- = E \setminus E^*$ the set containing all the removed edges. 
The variant DWVC-E$^+$ considers the case that $E^- = \emptyset$, and DWVC-E$^-$ considers the case that $E^+ = \emptyset$.

The following theorem analyzes the performances of the two algorithms for DWVC-E$^+$, from two different views.
We remark that for an instance $\{G = (V,E,W),Y_{\textup{orig}},E^+\}$ of \mbox{\textup{DWVC-E}}$^+$, $|E^+| = D$, and for each edge $e \in E^+$, $Y_{\textup{orig}}(e)$ and $\sigma(e)$ are initialized as 0 and 1, respectively.

\begin{theorem}
\label{theo:(1+1) EA for DWVC-E+}
The expected runtime of the \mbox{\textup{(1+1) EA}} (or \mbox{\textup{RLS}}) for \mbox{\textup{DWVC-E}}$^+$ is 
$\Or \big(\alpha m \log_{\alpha} W_{\textup{max}} \cdot \min\{D,\log (\alpha D \cdot \log_{\alpha} W_{\textup{max}})\} \! \big)$.
\end{theorem}

\begin{proof}
We first consider the expected runtime of the (1+1) EA to obtain an MFDS of $G^* = (V,E \cup E^+,W)$, starting with the given MFDS $Y_{\textup{orig}}$ of $G = (V,E,W)$.
Observe that $Y_{\textup{orig}}$ is a feasible dual-solution of $G^*$. Thus combining Lemma 1 and the general idea of the algorithm, we have that all  mutations accepted by the algorithm increase the LP values of the edges in $G^*$. 
If $Y_{\textup{orig}}$ is an MFDS of $G^*$, then any mutation on $Y_{\textup{orig}}$ results in an infeasible solution that would be rejected, i.e., the algorithm keeps the dual-solution $Y_{\textup{orig}}$ forever.

In the following discussion, we assume that $Y_{\textup{orig}}$ is not an MFDS of $G^*$. Observe that any increment on the LP values of the edges in $E$ would result in an infeasible solution that cannot be accepted by the algorithm. Thus we have that $Y^*(e) = Y_{\textup{orig}}(e)$ for each edge $e \in E$, and $Y^*(e) \geq Y_{\textup{orig}}(e)$ for each edge $e \in E^+$, where $Y^*$ is an MFDS of $G^*$ obtained by the (1+1) EA starting with $Y_{\textup{orig}}$.
To study the expected runtime of the (1+1) EA to get $Y^*$, two analytical ways from different views are given below: One considers the edges in $E^+$ sequentially; the other one considers that in an interleaved way.

We start with the analysis from the view that considers the edges in $E^+$ sequentially.
Let $e^* = [v_1,v_2]$ be an arbitrary edge in $E^+$ with $Y^*(e^*) - Y_{\textup{orig}}(e^*) > 0$. Since $Y^*(e^*) - Y_{\textup{orig}}(e^*) \leq W_{\textup{max}}$ and the fact that the step size of $e^*$ is initialized with value 1, Lemma~\ref{lem:(1+1) EA for increment of G1} gives that the (1+1) EA takes expected runtime 
$$\Or \! \left(\alpha m \log_{\alpha} \left(Y^*(e^*) - Y_{\textup{orig}}(e^*)\! \right) \! \right) = \Or (\alpha m \log_{\alpha} W_{\textup{max}})$$ 
to increase the LP value of $e^*$ from $Y_{\textup{orig}}(e^*)$ to $Y^*(e^*)$.
Combining the fact that the number of edges in $E^+$ is bounded by $D$, we have that the (1+1) EA takes expected runtime $\Or (\alpha m D \log_{\alpha} W_{\textup{max}})$ to get $Y^*$.

Now we analyze the expected runtime of the (1+1) EA to get $Y^*$ from the other view that considers the edges in $E^+$ as a whole.
For each edge $e \in E^+$, denote $Y^*(e) - Y_{\textup{orig}}(e)$ by $\Delta(e)$, and denote by $\beta(e)$ the number of valid mutations on $e$ that the algorithm requires to increase the LP value of $e$ from $Y_{\textup{orig}}(e)$ to $Y^*(e)$. 
Let $E_{\Delta} = \{e \in E^+ | \Delta(e) \neq 0\}$, and let the potential of the dual-solution $Y_{\textup{orig}}$ be 
\begin{eqnarray}
\label{eqn:61}
g(Y_{\textup{orig}}) = \sum_{e \in E_{\Delta}} \beta(e) \ .
\end{eqnarray}
Observe that $E_{\Delta} \subset E^+$. 
Since there may exist a mutation that is not only a valid mutation on $e_1 \in E^+$, but also a valid mutation on $e_2 \in E^+ \setminus \{e_1\}$, $g(Y_{\textup{orig}})$ is the upper bound of the number of valid mutations on the edges in $E_{\Delta}$ that the algorithm requires to get $Y^*$ starting from $Y_{\textup{orig}}$.
Moreover, the analysis of Lemma~\ref{lem:(1+1) EA for increment of G1} gives that any mutation on $Y_{\textup{orig}}$ cannot increase its potential.

To obtain the expected drift of $g$, we first consider the relation between $|E_{\Delta}|$ and $g(Y_{\textup{orig}})$.
For each edge $e \in E_{\Delta}$, Lemma~\ref{lem:(1+1) EA for increment of G1} gives that
\begin{eqnarray}
\label{eqn:62}
\beta(e) \leq 2 \alpha \left\lfloor \log_{\alpha} \big( \! (\alpha-1) \cdot \Delta(e) +1 \big)\right\rfloor
\leq 2 \alpha  \log_{\alpha} \big(\alpha \cdot \Delta(e) \! \big)
\leq 2 \alpha \left(\log_{\alpha} W_{\textup{max}} + 1 \right) \ .
\end{eqnarray}
By Equations~\ref{eqn:61} and~\ref{eqn:62}, we have that
$$|E_{\Delta}| \ge  \frac{g(Y_{\textup{orig}})}{2 \alpha \left(\log_{\alpha} W_{\textup{max}} + 1 \right)} \ \ \textup{and} \ \ g(Y_{\textup{orig}}) \le D \cdot 2 \alpha  \left(\log_{\alpha} W_{\textup{max}} + 1 \right) \ .$$
A valid mutation that chooses exactly one of the edge in $E_{\Delta}$ can be generated by the algorithm with probability $\Om(|E_{\Delta}|/(e \cdot m))$, which results in a new solution $Y'$ with $g(Y') = g(Y_{\textup{orig}}) -1$.
Thus the expected drift of $g$ can be bounded by 
$$\frac{|E_{\Delta}|}{e \cdot m} \ge \frac{g(Y_{\textup{orig}})}{e \cdot 2 \alpha m \cdot  \left(\log_{\alpha} W_{\textup{max}} + 1 \right)} \ .$$

As mentioned above, the maximum value that $g(Y_{\textup{orig}})$ can take is $2 \alpha D \left(\log_{\alpha} W_{\textup{max}} + 1 \right)$.
Combining it, the obvious minimum value 1 that $g(Y_{\textup{orig}})$ can take, and the expected drift of $g$, the Multiplicative Drift Theorem~\cite{algorithmica/DoerrJW12} gives that the (1+1) EA takes expected runtime 
$\Or \big(\alpha m \log_{\alpha} W_{\textup{max}} \cdot \log (\alpha D \cdot \log_{\alpha} W_{\textup{max}}) \! \big)$ to get $Y^*$.

Summarizing the above analysis, we can conclude that the (1+1) EA takes expected runtime $\Or \big(\alpha m \log_{\alpha} W_{\textup{max}} \cdot \min\{D,\log (\alpha D \cdot \log_{\alpha} W_{\textup{max}})\} \! \big)$ to find an MFDS of $G^*$ starting with $Y_{\textup{orig}}$.
Since we only consider the mutations selecting exactly one edge in the analysis for the (1+1) EA, the above conclusions also apply to the RLS.  
\end{proof}

Given an instance $\{G = (V,E,W),Y_{\textup{orig}},E^-\}$ of \mbox{\textup{DWVC-E}}$^-$, the following theorem considers the expected runtime of the (1+1) EA and RLS to obtain an \mbox{\textup{MFDS}} of $G^* = (V,E \setminus E^-,W)$, starting with the MFDS $Y_{\textup{orig}}$ of $G$. 
Note that the domain of definition for $Y_{\textup{orig}}$ and the weight function $W$ are modified as $E \setminus E^-$ after the removal, and $|E^-| = D$.

\begin{theorem}
\label{theo:(1+1) EA for DWVC-E-}
The expected runtime of the \mbox{\textup{(1+1) EA}} (or \mbox{\textup{RLS}}) for \mbox{\textup{DWVC-E}}$^-$ is 
$\Or \big(\alpha m \log_{\alpha} W_{\textup{max}} \cdot \log (\max\{\alpha m, \alpha D \cdot W_{\textup{max}} \}) \! \big)$.
\end{theorem}

\begin{proof}
Observe that $Y_{\textup{orig}}$ is a feasible dual-solution of $G^*$, and the endpoints of the edges in $E^-$ may not be tight with respect to $Y_{\textup{orig}}$ once the edges in $E^-$ are removed. 
Thus the LP values of the edges in $E_{G}(E^-)$ may have the room to be increased.
If $Y_{\textup{orig}}$ is an MFDS of $G^*$, then any mutation of the \mbox{\textup{(1+1) EA}} (or \mbox{\textup{RLS}}) on $Y_{\textup{orig}}$ would be rejected, and the algorithm keeps the dual-solution $Y_{\textup{orig}}$ forever.
In the following discussion, we assume that $Y_{\textup{orig}}$ is not an MFDS of $G^*$.

Let $Y^*$ be an arbitrary MFDS of $G^*$ obtained by the (1+1) EA (or RLS) starting with $Y_{\textup{orig}}$.
The above analysis gives that $Y^*(e) = Y_{\textup{orig}}(e)$ for each edge $e \in E \setminus \big( E^- \cup E_{G}(E^-) \! \big)$, and $Y^*(e) \geq Y_{\textup{orig}}(e)$ for each edge $e \in E_{G}(E^-)$.
Observe that all the edges in $E_{G}(E^-)$ are incident to the endpoints of the edges in $E^-$, and the number of endpoints of the edges in $E^-$ is upper bounded by $2D$. Combining the observation with the fact that the sum of the LP values of the edges sharing an endpoint cannot be larger than the weight of the endpoint under the dual-LP constraint, we have that $\sum_{e \in E_{G}(E^-)} Y^*(e)$ can be upper bounded by $2 D \cdot W_{\textup{max}}$.

For each edge $e \in E_{G}(E^-)$, denote $Y^*(e) - Y_{\textup{orig}}(e)$ by $\Delta(e)$, and denote by $\beta(e)$ the number of valid mutations on $e$ that the algorithm requires to increase the LP value of $e$ from $Y_{\textup{orig}}(e)$ to $Y^*(e)$.
Let $E_{\Delta} = \{e \in E_{G}(E^-) | \Delta(e) \neq 0\}$. Then we have 
$$\sum_{e \in E_{\Delta}} \Delta(e) \ = \sum_{e \in E_{G}(E^-)}\left(Y^*(e) - Y_{\textup{orig}}(e) \! \right) \ \leq \sum_{e \in E_{G}(E^-)} Y^*(e) \le 2 D \cdot W_{\textup{max}} \ .$$
Let the potential of the solution $Y_{\textup{orig}}$ be 
$$g(Y_{\textup{orig}}) = \sum_{e \in E_{\Delta}} \beta(e) \ .$$
Similar to the analysis given in Theorem~\ref{theo:(1+1) EA for DWVC-E+}, we have that $g(Y_{\textup{orig}})$ is the upper bound of the number of valid mutations on the edges in $E_{\Delta}$ that the algorithm requires to get $Y^*$ starting from $Y_{\textup{orig}}$, and any mutation on $Y_{\textup{orig}}$ cannot increase its potential.
The analysis for the expected drift of $g$ is divided into two cases, based on the value of $\sum_{e \in E_{\Delta}} \Delta(e)/|E_{\Delta}|$.

Case (1). $\sum_{e \in E_{\Delta}} \Delta(e) < \alpha \cdot |E_{\Delta}|$.
Lemma~\ref{lem:(1+1) EA for increment of G1} gives that $
\beta(e) \leq 2 \alpha (\log_{\alpha} \Delta(e) + 1)$ for each edge $e \in E_{\Delta}$. Thus we have
\begin{eqnarray*}
g(Y_{\textup{orig}}) = \sum_{e \in E_{\Delta}} \beta(e) 
&\le& 2 \alpha \cdot |E_{\Delta}|  + 2 \alpha \cdot \log_{\alpha} \left(\prod_{e \in E_{\Delta}} \Delta(e) \right) \\
&\le& 2 \alpha \cdot |E_{\Delta}|  + 2 \alpha \cdot |E_{\Delta}| \cdot \log_{\alpha} \frac{\sum_{e \in E_{\Delta}} \Delta(e)}{|E_{\Delta}|} \le 4 \alpha \cdot |E_{\Delta}| \le 4 \alpha m \ ,
\end{eqnarray*} 
implying that $|E_{\Delta}| \ge g(Y_{\textup{orig}}) / (4 \alpha)$.
A valid mutation that chooses exactly one of the edges in $E_{\Delta}$ can be generated by the algorithm with probability $\Om(|E_{\Delta}|/(e \cdot m))$, which results in a new solution $Y'$ with $g(Y') = g(Y_{\textup{orig}}) -1$.
Thus the expected drift of $g$ can be bounded by 
$$\frac{|E_{\Delta}|}{e \cdot m} \ge \frac{g(Y_{\textup{orig}})}{e \cdot 4 \alpha m} \ .$$

Case (2). $\sum_{e \in E_{\Delta}} \Delta(e) \ge \alpha \cdot |E_{\Delta}|$. By Lemma~\ref{lem:(1+1) EA for increment of G1}, we can get that 
\begin{eqnarray*}
g(Y_{\textup{orig}}) = \sum_{e \in E_{\Delta}} \beta(e) \le \sum_{e \in E_{\Delta}} 2 \alpha (\log_{\alpha} \Delta(e) + 1) 
\le |E_{\Delta}| \cdot 2 \alpha (\log_{\alpha} W_{\textup{max}} +1) \ ,
\end{eqnarray*} 
implying that
$|E_{\Delta}| \ge  \frac{g(Y_{\textup{orig}})}{2 \alpha \left(\log_{\alpha} W_{\textup{max}} + 1 \right)}$.
Using the reasoning similar to that given for Case (1), we have that the expected drift of $g$ can be bounded by 
$$\frac{|E_{\Delta}|}{e \cdot m} \ge \frac{g(Y_{\textup{orig}})}{e \cdot 2 \alpha m \cdot  \left( \log_{\alpha} W_{\textup{max}} + 1 \right)} \ .$$
Now we consider the maximum value that $g(Y_{\textup{orig}})$ can take,
\begin{eqnarray}
\label{ee1}
g(Y_{\textup{orig}}) = \sum_{e \in E_{\Delta}} \beta(e) 
&\le& 2 \alpha \cdot |E_{\Delta}|  + 2 \alpha \cdot \log_{\alpha} \left(\prod_{e \in E_{\Delta}} \Delta(e) \right) \\
\label{ee2}
&\le& 2 \alpha \cdot |E_{\Delta}|  + 2 \alpha \cdot |E_{\Delta}| \cdot \log_{\alpha} \frac{\sum_{e \in E_{\Delta}} \Delta(e)}{|E_{\Delta}|} \\
\label{ee3}
&\le& 4 \alpha \cdot |E_{\Delta}| \cdot \log_{\alpha} \frac{\sum_{e \in E_{\Delta}} \Delta(e)}{|E_{\Delta}|} \\
\label{ee4}
&\le& 4 \alpha \cdot |E_{\Delta}| \cdot \log_{\alpha} \frac{2D \cdot W_{\textup{max}}}{|E_{\Delta}|} \\
\label{ee5}
&\le& 4 \alpha \cdot \frac{2D \cdot W_{\textup{max}}}{e} \cdot \log_{\alpha} e \\
\label{ee6}
&\le& 8 \alpha D \cdot W_{\textup{max}} \ ,
\end{eqnarray} 
where the factor $\log_{\alpha} e$ is not greater than $e$ as $\alpha \in [2, W_{\textup{max}}]$, and Inequality~\ref{ee5} can be derived by the observation that $f(x) = x \cdot \log_{\alpha} (2D \cdot W_{\textup{max}}/x)$ ($x > 0$) gets its maximum value when $x = 2D \cdot W_{\textup{max}}/e$.

Summarizing the analysis for Cases (1-2), we have that the expected drift of $g$ can be bounded by 
$$\frac{|E_{\Delta}|}{e \cdot m} \ge \frac{g(Y_{\textup{orig}})}{e \cdot 2 \alpha m \cdot  \max\{2, \log_{\alpha} W_{\textup{max}} + 1\} } = \frac{g(Y_{\textup{orig}})}{e \cdot 2 \alpha m \cdot (\log_{\alpha} W_{\textup{max}} + 1 )} \ ,$$
and the maximum value of $g(Y_{\textup{orig}})$ can be bounded by 
$$\max \{4 \alpha m, 8 \alpha D \cdot W_{\textup{max}} \} \ .$$

The Multiplicative Drift Theorem~\cite{algorithmica/DoerrJW12} implies that the (1+1) EA takes expected runtime
$\Or \big(\alpha m \log_{\alpha} W_{\textup{max}} \cdot \log (\max\{\alpha m, \alpha D \cdot W_{\textup{max}} \}) \! \big)$ to find an MFDS of $G^*$ starting with $Y_{\textup{orig}}$.
Since we only consider the mutations selecting exactly one edge in the analysis for the (1+1) EA, the above conclusions also apply to the RLS.   
\end{proof}

Consider an instance $\{G = (V,E,W),Y_{\textup{orig}},E^*\}$ of \mbox{\textup{DWVC-E}}. By the analysis for \mbox{\textup{DWVC-E}}$^+$ and \mbox{\textup{DWVC-E}}$^-$, we have that $Y_{\textup{orig}}$ is a feasible dual-solution of $G^* = (V,E^*,W)$, and the LP values of the edges in $E^+ \cup E_{G}(E^-)$ may have the room to be increased, where
$E^+ = E^* \setminus E$ and $E^- = E \setminus E^*$. 
Since $|E^+ \cup E^-|$ is bounded by $D$, we can derive the following theorem for \mbox{\textup{DWVC-E}} using the reasoning similar to that for Theorems~\ref{theo:(1+1) EA for DWVC-E+} and~\ref{theo:(1+1) EA for DWVC-E-}.

\begin{theorem}
\label{theo:(1+1) EA for DWVC-E}
The expected runtime of the \mbox{\textup{(1+1) EA}} (or \mbox{\textup{RLS}}) for \mbox{\textup{DWVC-E}} is 
$\Or \big(\alpha m \log_{\alpha} W_{\textup{max}} \cdot \log (\max\{\alpha m, \alpha D \cdot W_{\textup{max}} \}) \! \big)$.
\end{theorem}

\subsection{Analysis for DWVC with Weight Modification}

We start the subsection with the analysis of the (1+1) EA and RLS for the two special variants of DWVC-W, namely, DWVC-W$^+$ and DWVC-W$^-$. 
Denote by $V^+$ the set containing all the vertices $v$ with $W^*(v) > W(v)$, and by $V^-$ the set containing all the vertices $v$ with $W^*(v) < W(v)$. 
The variant DWVC-W$^+$ considers the case that $V^- = \emptyset$, and DWVC-W$^-$ considers the case that $V^+ = \emptyset$.

Consider an instance $\{G = (V,E,W),Y_{\textup{orig}},W^+,V^+\}$ of \mbox{\textup{DWVC-W}}$^+$. 
Observe that $Y_{\textup{orig}}$ is an obviously feasible dual-solution of $G^* = (V,E,W^+)$, and the LP values of the edges in $E_{G^*}(V^+)$ may have the room to be increased if $Y_{\textup{orig}}$ is not an MFDS of $G^*$.
The following lemma shows that the sum of the feasible LP value increments on the edges in $E_{G^*}(V^+)$ can be upper bounded, as these edges are all incident to the vertices in $V^+$.

\begin{lemma}
\label{lem:weight bound}
For any \mbox{\textup{MFDS}} $Y^*$ obtained by the \mbox{\textup{(1+1) EA}} (or \mbox{\textup{RLS}}) for the instance $\{G = (V,E,W),Y_{\textup{orig}},W^+,V^+\}$ of \mbox{\textup{DWVC-W}}$^+$,
\begin{equation*}
    \sum_{e \in E} \left( Y^*(e) - Y_{\textup{orig}}(e) \! \right) \leq \sum_{v \in V^+} \left ( W^+(v) - W(v) \! \right) \leq D \cdot W_{\textup{max}} \ .
\end{equation*}
\end{lemma}

\begin{proof}
Since $Y_{\textup{orig}}$ is a feasible dual-solution of $G^*$, by Lemma~\ref{lem:sign function}, $Y^*(e) \geq Y_{\textup{orig}}(e)$ for each edge $e \in E$. Let $E_{W^+}$ be the set containing all the edges $e \in E$ with $Y^*(e) > Y_{\textup{orig}}(e)$. Then we have
\begin{equation}
\label{e1}
\sum_{e \in E} \left( Y^*(e) - Y_{\textup{orig}}(e) \! \right) = \sum_{e \in E_{W^+}} \left(Y^*(e) - Y_{\textup{orig}}(e) \! \right) \ .
\end{equation}
Note that the LP values of the edges in $E \setminus E_{G^*}(V^+)$ cannot be increased, thus $E_{W^+} \subseteq E_{G^*}(V^+)$.

For each edge $e \in E_{W^+}$, let $\tau(e)$ be the endpoint of $e$ that is tight with respect to $Y_{\textup{orig}}$ (if both endpoints of $e$ are tight, then arbitrarily choose one as $\tau(e)$). Observe that $\tau(e) \in V^+$ for each edge $e \in E_{W^+}$; otherwise, the LP value of the edge cannot be increased under the dual-LP constraint. Thus for any vertex $v \in V^+$, we have
\begin{equation}
\label{e2}
\sum_{e \in E_{W^+}| \tau(e) = v} \left( Y^*(e) - Y_{\textup{orig}}(e) \! \right) \leq W^+(v) - W(v) \ .
\end{equation}
Then summarizing Inequality~(\ref{e2}) over all vertices in $V^+$, we can get
\begin{equation}
\label{e33}
\sum_{e \in E_{W^+}} \left( Y^*(e) - Y_{\textup{orig}}(e) \! \right) \leq \sum_{v \in V^+} \left ( W^+(v) - W(v) \! \right) \leq D \cdot W_{\textup{max}} \ .
\end{equation}

Combining Equality~(\ref{e1}) and Inequality~(\ref{e33}) gives the claimed inequality.  
\end{proof}

Using the reasoning similar to that for Theorem~\ref{theo:(1+1) EA for DWVC-E-} and the upper bound given by Lemma~\ref{lem:weight bound}, we can get the following theorem for \mbox{\textup{DWVC-W}}$^+$.

\begin{theorem}
\label{theo:(1+1) EA for DWVC-W+}
The expected runtime of the \mbox{\textup{(1+1) EA}} (or \mbox{\textup{RLS}}) for \mbox{\textup{DWVC-W}}$^+$ is $\Or \big(\alpha m \log_{\alpha} W_{\textup{max}} \cdot \log (\max\{\alpha m, \alpha D \cdot W_{\textup{max}} \}) \! \big)$.
\end{theorem}

Given an instance $\{G = (V,E,W),Y_{\textup{orig}},W^-,V^-\}$ of \mbox{\textup{DWVC-W}}$^-$, if $Y_{\textup{orig}}$ is a feasible dual-solution of $G^* = (V,E,W^-)$, then $Y_{\textup{orig}}$ is still an MFDS of $G^*$.
Otherwise, we have to first decrease the LP values of the edges violating the dual-LP constraint with respect to the maintained dual-solution, as the general idea of the algorithms given in Section~\ref{sec:algorithms}, to get the first feasible dual-solution as soon as possible. The remaining analysis to get an MFDS of $G^*$ based on the first feasible dual-solution is similar to that given for DWVC-E$^+$.

\begin{theorem}
\label{theo:(1+1) EA for DWVC-W-}
The expected runtime of the \mbox{\textup{(1+1) EA}} (or \mbox{\textup{RLS}}) for \mbox{\textup{DWVC-W}}$^-$ is 
$\Or \big(\alpha m \log_{\alpha} W_{\textup{max}} \cdot \log (\max\{\alpha m, \alpha D \cdot W_{\textup{max}} \}) \! \big)$.
\end{theorem}

\begin{proof}
We first analyze the expected runtime of the (1+1) EA to obtain an \mbox{\textup{MFDS}} $Y^*$ of $G^* = (V,E,W^-)$, starting with the MFDS $Y_{\textup{orig}}$ of $G$. For the soundness and completeness of the proof, we assume that $Y_{\textup{orig}}$ is an infeasible dual-solution of $G^*$.
Then the whole process can be divided into Phase (I) and Phase (II). Phase (I) contains all steps of the algorithm until it finds the first feasible dual-solution $Y_t$ of $G^*$; Phase (II) follows Phase (I), which contains all steps of the algorithm until it obtains the MFDS $Y^*$ of $G^*$. 

{\bf Phase (I)}.
By Lemma~\ref{lem:infeasible to MFDS}, to get the first feasible dual-solution $Y_t$ of $G^*$, the (1+1) EA only can decrease the LP values of the edges in $E_{G^*}(Y_{\textup{orig}})$, where $E_{G^*}(Y_{\textup{orig}}) \subseteq E_{G^*}(V^-)$. Thus $Y_t(e) \leq Y_{\textup{orig}}(e)$ for each edge $e \in E_{G^*}(Y_{\textup{orig}})$, and $Y_t(e) = Y_{\textup{orig}}(e)$ for each edge $e \in E \setminus E_{G^*}(Y_{\textup{orig}})$.
Denote $Y_{\textup{orig}}(e) - Y_t(e)$ by $\Delta(e)$ for each edge $e \in E_{G^*}(Y_{\textup{orig}})$, and denote by $\beta(e)$ the number of valid mutations on $e$ that the (1+1) EA requires to decrease the LP value of $e$ from $Y_{\textup{orig}}(e)$ to $Y_t(e)$. Let $E_{\Delta} = \{e \in E_{G^*}(Y_{\textup{orig}}) \ | \ \Delta(e) \neq 0\}$. Since each edge in $E_{\Delta}$ has an endpoint that is in $V^-$, 
$\sum_{e \in E_{\Delta}} \Delta(e) \le \sum_{e \in E_{\Delta}} Y_{\textup{orig}}(e) \le D \cdot W_{\textup{max}}$. 
Let the potential of the solution $Y_{\textup{orig}}(e)$ be 
$$g\left(Y_{\textup{orig}}(e) \! \right) = \sum_{e \in E_{\Delta}} \beta(e) \ .$$
Similar to the analysis given in Theorem~\ref{theo:(1+1) EA for DWVC-E+}, we have that $g(Y_{\textup{orig}})$ is the upper bound of the number of valid mutations on the edges in $E_{\Delta}$ that the algorithm requires to get $Y_t$ starting from $Y_{\textup{orig}}$, and any mutation on $Y_{\textup{orig}}$ cannot increase its potential.
The analysis for the expected drift of $g$ is divided into two cases, based on the value of $\sum_{e \in E_{\Delta}} \Delta(e)/|E_{\Delta}|$.

Case (1). $\sum_{e \in E_{\Delta}} \Delta(e) < \alpha \cdot |E_{\Delta}|$.
By Lemma~\ref{lem:(1+1) EA for decrement of G1}, we have 
$$\beta(e) \leq \lceil \log_{\alpha} \big( (\alpha-1) \Delta(e) +1 \big) \rceil
\leq \lceil \log_{\alpha} \big( \alpha \Delta(e) \! \big) \rceil
\leq \lceil \log_{\alpha} \Delta(e) + 1 \rceil 
\leq \log_{\alpha} \Delta(e) + 2$$
for each edge $e \in E_{\Delta}$. Thus
\begin{eqnarray*}
g(Y_{\textup{orig}}) = \sum_{e \in E_{\Delta}} \beta(e) 
&\le& 2 |E_{\Delta}| + \log_{\alpha} \left(\prod_{e \in E_{\Delta}} \Delta(e) \right) \\
&\le& 2 |E_{\Delta}|  + |E_{\Delta}| \cdot \log_{\alpha} \frac{\sum_{e \in E_{\Delta}} \Delta(e)}{|E_{\Delta}|} \le 3 |E_{\Delta}| \ ,
\end{eqnarray*} 
implying that $|E_{\Delta}| \ge g(Y_{\textup{orig}}) / 3$, and the maximum value that $g(Y_{\textup{orig}})$ can take is $3 m$ (as $|E_{\Delta}| \le m$).
A valid mutation that chooses exactly one of the edge in $E_{\Delta}$ can be generated by the algorithm with probability $\Om(|E_{\Delta}|/(e \cdot m))$, which results in a new solution $Y'$ with $g(Y') = g(Y_{\textup{orig}}) -1$.
Consequently, the expected drift of $g$ can be bounded by 
$$\frac{|E_{\Delta}|}{e \cdot m} \ge \frac{g(Y_{\textup{orig}})}{3e \cdot m} \ .$$

Case (2). $\sum_{e \in E_{\Delta}} \Delta(e) \ge \alpha \cdot |E_{\Delta}|$. 
By Lemma~\ref{lem:(1+1) EA for decrement of G1}, we have that 
\begin{eqnarray*}
g(Y_{\textup{orig}}) = \sum_{e \in E_{\Delta}} \beta(e) \le \sum_{e \in E_{\Delta}} (\log_{\alpha} \Delta(e) + 2)
\le |E_{\Delta}| \cdot (\log_{\alpha} W_{\textup{max}} +2) \ ,
\end{eqnarray*} 
implying that
$|E_{\Delta}| \ge  \frac{g(Y_{\textup{orig}})}{\log_{\alpha} W_{\textup{max}} + 2}$.
Using the reasoning similar to that given for Case (1), we have that the expected drift of $g$ can be bounded by 
$$\frac{|E_{\Delta}|}{e \cdot m} \ge \frac{g(Y_{\textup{orig}})}{e \cdot m \cdot  \left( \log_{\alpha} W_{\textup{max}} + 2 \right)} \ .$$
Now we consider the maximum value that $g(Y_{\textup{orig}})$ can take,
\begin{eqnarray}
\label{ee7}
g(Y_{\textup{orig}}) = \sum_{e \in E_{\Delta}} \beta(e) 
&\le& 2 |E_{\Delta}|  + \log_{\alpha} \left(\prod_{e \in E_{\Delta}} \Delta(e) \right) \\
\label{ee8}
&\le& 2 |E_{\Delta}|  + |E_{\Delta}| \cdot \log_{\alpha} \frac{\sum_{e \in E_{\Delta}} \Delta(e)}{|E_{\Delta}|} \\
\label{ee9}
&\le& 3 |E_{\Delta}| \cdot \log_{\alpha} \frac{\sum_{e \in E_{\Delta}} \Delta(e)}{|E_{\Delta}|} \\
\label{ee10}
&\le& 3 |E_{\Delta}| \cdot \log_{\alpha} \frac{D \cdot W_{\textup{max}}}{|E_{\Delta}|} \\
\label{ee11}
&\le& 3 \cdot \frac{D \cdot W_{\textup{max}}}{e} \cdot \log_{\alpha} e \\ 
\label{ee12}
&\le& 3 D \cdot W_{\textup{max}} \ ,
\end{eqnarray} 
where the factor $\log_{\alpha} e$ is not greater than $e$ as $\alpha \in [2, W_{\textup{max}}]$, and Inequality~\ref{ee11} can be derived by the observation that $f(x) = x \cdot \log_{\alpha} (D \cdot W_{\textup{max}}/x)$ ($x > 0$) gets its maximum value when $x = D \cdot W_{\textup{max}}/e$. 

Summarizing the analysis for Cases (1-2) and the fact that the value of $\alpha$ is between $2$ and $W_{\textup{max}}$, we have that the expected drift of $g$ can be lower bounded by 
$$\frac{|E_{\Delta}|}{e \cdot m} \ge \frac{g(Y_{\textup{orig}})}{e \cdot m \cdot  \max\{3, \log_{\alpha} W_{\textup{max}} + 2 \} } \ge \frac{g(Y_{\textup{orig}})}{e \cdot m \cdot (\log_{\alpha} W_{\textup{max}} + 2)} \ ,$$
and the maximum value that $g(Y_{\textup{orig}})$ can take is upper bounded by
$$\max\{3m, 3D \cdot W_{\textup{max}}\}\ .$$
The Multiplicative Drift Theorem~\cite{algorithmica/DoerrJW12} implies that the (1+1) EA takes expected runtime
$\Or \big(m \log_{\alpha} W_{\textup{max}} \cdot \log (\max\{m, D \cdot W_{\textup{max}} \}) \! \big)$ to obtain the first feasible dual-solution $Y_t$ starting with $Y_{\textup{orig}}$.

{\bf Phase (II)}.
Obviously $Y_t$ may not be an MFDS of $G^* = (V,E,W^-)$. Thus we also need to consider the process of the (1+1) EA to get an MFDS of $G^*$ starting with $Y_t$. To simplify the analysis, we intend to transform Phase (II) as an execution of the (1+1) EA for an instance $\{G_t = (V,E,W_t), Y_t, W^-, V_t\}$ of DWVC-W$^+$. 
Thus in the following discussion, we first give the setting way of the weight function $W_t$ such that $Y_t$ is an MFDS of $G_t$, and $W_t(v) \le W^-(v)$ for each vertex $v \in V$.

Let $V_{\Delta}$ contain all endpoints of the edges in $E_{\Delta}$. For each vertex $v \in V \setminus V_{\Delta}$, let $W'(v) = W(v)$, and for each vertex $v \in V_{\Delta}$, let
$$W'(v) = W(v) - \sum_{e \in E_{\Delta}| e \cap v \neq \emptyset}\Delta(e) \ .$$
As $Y_{\textup{orig}}$ is an MFDS of $G$, $Y_t$ is an obvious MFDS of $G' = (V,E,W')$.
Note that there may exist some vertex $v \in V$ with $W'(v) > W^-(v)$. 
Thus for each vertex $v \in V$, we let 
\begin{eqnarray}
\label{eqn:11-1}
W_t(v) = \min\{W'(v), W^-(v)\} \ .
\end{eqnarray}

Because of Equality~\ref{eqn:11-1} and the fact that $Y_t$ is a feasible dual-solution of both $G'$ and $G^*$, $Y_t$ is a feasible dual-solution of $G_t$. Furthermore, since $Y_t$ is an MFDS of $G'$, $Y_t$ is an MFDS of $G_t$.
Now let $V_t$ contain all vertices $v \in V$ with $W_t(v) < W^-(v)$. 
Then the instance $\{G_t = (V,E,W_t), Y_t, W^-, V_t\}$ of DWVC-W$^+$ is completely constructed.

It is necessary to remark that for an edge $e$ in $E_{\Delta}$, the step size of $e$ may be larger than $Y^*(e) - Y_t(e)$ at the beginning of Phase (II), then Lemma~\ref{lem:(1+1) EA for increment of G1} is invalid under the situation. Fortunately, the step size of $e$ is at most $\alpha \cdot (Y_{\textup{orig}}(e) - Y_t(e)) = \alpha \cdot \Delta(e)$. Thus using the multiplicative drift analysis similar to that given above, we can get that the expected runtime of the (1+1) EA to decrease the step sizes of the edges in $E_{\Delta}$ to the feasible values is bounded by $\Or \big(m \log_{\alpha} W_{\textup{max}} \cdot \log (\max\{m, D \cdot W_{\textup{max}} \}) \! \big)$.

Now we assume that Lemma~\ref{lem:(1+1) EA for increment of G1} is valid for each edge in $E_{\Delta}$.  
Similar to Lemma~\ref{lem:weight bound}, we consider the upper bound on the sum of the feasible LP value increments on the edges with respect to $W^-$, where
\begin{eqnarray*}
\sum_{v \in V_t} \big( W^-(v) - W_t(v) \! \big) &\leq& \sum_{v \in V_t} \big( W(v) - W_t(v) \! \big) \\
&\leq& \sum_{v \in V} \Big( \! \left(W(v) - W^-(v) \! \right) + \left(W(v) - W'(v) \right) \! \Big) \\
&\leq&  \sum_{v \in V} \left( \left(W(v) - W^-(v) \! \right) + \sum_{e \in E_{\Delta}| e \cap v \neq \emptyset}\Delta(e) \right) \\ 
&\leq& \sum_{v \in V} \left(W(v) - W^-(v) \! \right) + 2 \sum_{e \in E_{\Delta}} \Delta(e) \\
&\leq& D \cdot W_{\textup{max}} + 2 D \cdot W_{\textup{max}} = 3 D \cdot W_{\textup{max}} \ .
\end{eqnarray*}
By Lemma~\ref{lem:weight bound} and the reasoning similar to that for Theorem~\ref{theo:(1+1) EA for DWVC-E-}, we have that the (1+1) EA takes expected runtime $\Or \big(\alpha m \log_{\alpha} W_{\textup{max}} \cdot \log (\max\{\alpha m, \alpha D \cdot W_{\textup{max}} \}) \! \big)$ to get an MFDS of $G^*$ starting with $Y_t$.

Summarizing the above discussion, the (1+1) EA takes expected runtime $\Or \big(\alpha m \log_{\alpha} W_{\textup{max}} \cdot \log (\max\{\alpha m, \alpha D \cdot W_{\textup{max}} \}) \! \big)$ to get an MFDS of $G^*$. The above expected runtime also applies to the RLS.   
\end{proof}

\begin{theorem}
\label{theo:(1+1) EA for DWVC-E}
The expected runtime of the \mbox{\textup{(1+1) EA}} (or \mbox{\textup{RLS}}) for \mbox{\textup{DWVC-W}} is 
$\Or \big(\alpha m \log_{\alpha} W_{\textup{max}} \cdot \log (\max\{\alpha m, \alpha D \cdot W_{\textup{max}} \}) \! \big)$.
\end{theorem}

\begin{proof}
Consider an instance $\{G = (V,E,W),Y_{\textup{orig}},W^*,V^+,V^-\}$ of \mbox{\textup{DWVC-W}}. By the discussion for \mbox{\textup{DWVC-W}}$^+$ and \mbox{\textup{DWVC-W}}$^-$, if $Y_{\textup{orig}}$ is a feasible dual-solution with respect to $G^*$, then only the LP values of the edges in $E_{G^*}(V^+)$ may have the room to be increased, and the sum of the increments can be upper bounded by $D \cdot W_{\textup{max}}$.
If $Y_{\textup{orig}}$ is an infeasible dual-solution with respect to $G^*$, then we have that the sum of the decrements on the LP values of the edges in $E_{G^*}(V^-)$ and the sum of the increments on the LP values of the edges incident to the vertices in $N_{G^*}(V^-) \cup V^-  \cup V^+$ can be upper bounded by $D \cdot W_{\textup{max}}$ and $3 D \cdot W_{\textup{max}}$, respectively. 
Using the reasoning similar to that for Theorems~\ref{theo:(1+1) EA for DWVC-E-} and~\ref{theo:(1+1) EA for DWVC-W-}, we have the theorem for \mbox{\textup{DWVC-W}}.
\end{proof}

\section{Runtime Analysis for the \mbox{\textup{RLS with 1/5-th Rule}} and \mbox{\textup{(1+1) EA with 1/5-th Rule}}}

Given an MFDS $Y^*$ that is obtained by the \mbox{\textup{RLS with 1/5-th Rule}} starting with a feasible dual-solution $Y^{\dagger}$ of $G^*$, the following two lemmata consider the behavior of the algorithm on a specific edge $e^*$ in $G^*$ during the process from $Y^{\dagger}$ to $Y^*$. 
Denote $Y^*(e^*) - Y^{\dagger}(e^*)$ by $\Delta$. 

\begin{lemma}
\label{lem:RLS with 1/5-th Rule increase weight partly}
If the step size of the edge $e^*$ has an initial value $\sigma_1 > 0$, then the \mbox{\textup{RLS with 1/5-th Rule}} takes expected runtime $\Or \left(m \left(\log_{\alpha} \sigma_1 + \log_{\alpha} \Delta \right) \! \right)$ to find a feasible dual-solution $Y^{\ddagger}$ such that $Y^{\ddagger}(e^*) - Y^{\dagger}(e^*) \geq \Delta /(\alpha + 1)$, during the process from $Y^{\dagger}$ to $Y^*$.
\end{lemma}

\begin{proof}
Since $Y^{\dagger}$ is a feasible dual-solution of $G^*$, any mutation that decreases the LP values of the edges would be rejected.
Thus for any dual-solution $Y$ accepted by the algorithm during the process from $Y^{\dagger}$ to $Y^*$, there is $Y(e^*) \geq Y^{\dagger}(e^*)$.
The following analysis divides the process from $Y^{\dagger}$ to $Y^{\ddagger}$ into two phases: Phase (I) and Phase (II). 

As the initial value $\sigma_1$ of the step size of $e^*$ may be larger than $\Delta$, Phase (I) contains all steps of the algorithm until the step size of $e^*$ is not greater than $\Delta$, where $Y_1$ denotes the dual-solution maintained by the algorithm at that moment. 
We remark that for any dual-solution $Y$ (including $Y_1$) obtained during Phase (I), $Y(e^*) = Y^{\dagger}(e^*)$.
Phase (II) follows Phase (I), and ends if the step size of $e^*$ is greater than $Y^*(e^*) - Y_2(e^*)$, where $Y_2$ denotes the dual-solution maintained by the algorithm at the moment. 
We will show that $Y_2$ satisfies the claimed condition given for $Y^{\ddagger}$. 
If $\sigma_1 \leq \Delta$, then we are already at Phase (II). 
For the soundness and completeness of the proof, we assume that $\sigma_1 > \Delta$ in the following discussion. 

Let $\sigma_1 = \alpha^p$ and $\alpha^q \leq \Delta < \alpha^{q+1/4}$, where $p, q \in \{l/4 | l \in \mathds{N} \}$. 
Now we analyze the expected runtime that Phase (I) takes to decrease the step size of $e^*$ from $\alpha^p$ to $\alpha^q$. 
Since $\sigma(e^*) > Y^*(e^*) -Y(e^*)$ always holds for any maintained dual-solution $Y$ during Phase (I), and the fact that $Y^{\dagger}$ is feasible, any mutation on $e^*$ cannot be accepted, and decreases the step size of $e^*$ by a factor $\alpha^{1/4}$. 
Thus Phase (I) needs $\Or(p-q) = \Or(p) = \Or(\log_{\alpha} \sigma_1)$ mutations on $e^*$ to decrease the step size to $\alpha^q$.  
 
Now we assume that the step size of $e^*$ is decreased to $\alpha^q$, i.e., we are at Phase (II) now. 
If a mutation on $e^*$ increases its LP value, then the mutation would be accepted, and the exponent of the step size would be increased to $q+1$; otherwise, the mutation would be rejected, and the exponent of the step size would be decreased to $q -1/4$.
The mutation on $e^*$ increases or decreases its LP value with the same probability 1/2, hence the value of the exponent increases by 1 or decreases by 1/4 with the same probability 1/2. 
Observe that the drift on the exponent is $(1- 1/4)/2 = 3/8$. 
If the step size of $e^*$ is increased to over $\alpha^{\lceil \log_{\alpha} \Delta \rceil}$ (at most $\alpha^{\lceil \log_{\alpha} \Delta \rceil + 3/4}$), then Phase (II) obviously ends.
In fact, the step size may not be increased to over $\alpha^{\lceil \log_{\alpha} \Delta \rceil}$ during Phase (II).
Using the Additive Drift Theorem~\cite{he2004study}, the algorithm needs $\Or(\log_{\alpha} \Delta - q) = \Or(\log_{\alpha} \Delta)$ mutations on $e^*$ to increase the exponent to over $\lceil \log_{\alpha} \Delta \rceil$. 
Thus Phase (II) contains $\Or(\log_{\alpha} \Delta)$ mutations on $e^*$. 

For the dual-solution $Y_2$ obtained by Phase (II), $\sigma(e^*) > Y^*(e^*) - Y_2(e^*)$. 
Since Phase (II) contains at least one mutation increasing the LP value of $e^*$ that is accepted, the gap $\Delta = Y^*(e^*) - Y^{\dagger}(e^*)$ is decreased by at least $\sigma(e^*)/\alpha$, and we have 
$$Y^*(e^*) - Y_2(e^*) 
\leq \Delta - \sigma(e^*)/\alpha 
\leq  \Delta - \left(Y^*(e^*) - Y_2(e^*) \! \right)/\alpha \ .$$
By the above inequality, it is easy to get that $Y_2(e^*) - Y^{\dagger}(e^*) \geq \Delta / (\alpha +1)$.

Summarizing the above analysis, the algorithm takes expected runtime $\Or \left(m (\log_{\alpha} \sigma_1 + \log_{\alpha} \Delta) \! \right)$ to get a feasible dual-solution satisfying the claimed condition. 
\end{proof}

\begin{lemma}
\label{lem:RLS with 1/5-th Rule increase weight}
If the step size of the edge $e^*$ has an initial value not greater than $\Delta$, then the \mbox{\textup{RLS with 1/5-th Rule}} takes expected runtime $\Or \left(\alpha m \log_{\alpha} \Delta \cdot \log \Delta \right)$ to increase the LP value of $e^*$ from $Y^{\dagger}(e^*)$ to $Y^*(e^*)$, during the process from $Y^{\dagger}$ to $Y^*$.
\end{lemma}

\begin{proof}
Let $Y$ be an arbitrary dual-solution obtained during the process from $Y^{\dagger}$ to $Y^*$. In the following discussion, we first analyze the expected runtime of the \mbox{\textup{RLS with 1/5-th Rule}} to obtain a solution $Y'$ starting with $Y$ such that $Y'(e^*) - Y(e^*) \geq \left(Y^*(e^*) - Y(e^*) \! \right) /(\alpha + 1)$. 

As the initial value of the step size of $e^*$ is not greater than $\Delta$, we have the observation that the maximum value of the step size of $e^*$ during the process from $Y^{\dagger}$ to $Y^*$ is at most $\alpha^{p + 1}$, where $p = \lceil \log_{\alpha} \Delta \rceil$. 
For the moment that the algorithm maintains the dual-solution $Y$, the corresponding step size of $e^*$ may be greater than $Y^*(e^*) - Y(e^*)$, thus we have to consider Phase (I) (defined in the proof of Lemma~\ref{lem:RLS with 1/5-th Rule increase weight partly}). 
Since the step size of $e^*$ is bounded by $\alpha^{p + 1}$, Phase (I) needs $\Or(\log_{\alpha} \Delta)$ mutations on $e^*$. 
Furthermore, since $Y(e^*) \geq Y^{\dagger}(e^*)$, Phase (II) needs $\Or \left(\log_{\alpha} \left(Y^*(e^*) - Y(e^*) \! \right) \! \right) = \Or(\log_{\alpha} \Delta)$ mutations on $e^*$. Consequently, the \mbox{\textup{RLS with 1/5-th Rule}} takes expected runtime $\Or(m \log_{\alpha} \Delta)$ to get $Y'$ starting with $Y$. 

Using the above conclusion, the Multiplicative Drift Theorem~\cite{algorithmica/DoerrJW12} implies the claimed runtime.
\end{proof}

Now we analyze the expected runtime of the \mbox{\textup{RLS with 1/5-th Rule}} to make the edge $e^*$ satisfy the dual-LP constraint, if it starts with an infeasible dual-solution with respect to which $e^*$ violates the dual-LP constraint.

\begin{lemma}
\label{lem:RLS with 1/5-th Rule decrease weight}
Consider an infeasible dual-solution $Y^{\dagger}$ of $G^*$, with respect to which the edge $e^*$ violates the dual-LP constraint. For the first feasible dual-solution $Y^{\ddagger}$ obtained by the \mbox{\textup{RLS with 1/5-th Rule}} starting with $Y^{\dagger}$, the algorithm takes expected runtime $\Or \! \left(m \log_{\alpha} \left(Y^{\dagger}(e^*) - Y^{\ddagger}(e^*) \! \right) \! \right)$
 to decrease the LP value of $e^*$ from $Y^{\dagger}(e^*)$ to $Y^{\ddagger}(e^*)$. 
\end{lemma}
\begin{proof} 
Assume that the step size of $e^*$ has an initial value $\alpha^q$, where $q \geq 0$. 
Since $Y^{\dagger}$ is an infeasible dual-solution of $G^*$, if a mutation on $e^*$ decreases its LP value, then the mutation is accepted, and the exponent of the step size of $e^*$ is increased by $1$; otherwise, the mutation is rejected, and the exponent of the step size is decreased by $1/4$.
Observe that the mutation on $e^*$ increases or decreases its LP value with the same probability 1/2. 
Thus we have that the drift on the exponent of the step size of $e^*$ is $(1- 1/4)/2 = 3/8$, during the process that decreases the LP value of $e^*$ from $Y^{\dagger}(e^*)$ to $Y^{\ddagger}(e^*)$. 

If the step size of $e^*$ is increased to not less than $\alpha^{\lceil \log_{\alpha} (Y^{\dagger}(e^*) - Y^{\ddagger}(e^*)) \rceil +1}$ during the process that decreases the LP value of $e^*$ from $Y^{\dagger}(e^*)$ to $Y^{\ddagger}(e^*)$, then the LP value of $e^*$ is decreased to 0. 
In fact, the step size may not be increased to over $\alpha^{\lceil \log_{\alpha} (Y^{\dagger}(e^*) - Y^{\ddagger}(e^*)) \rceil +1}$ during the process, because $\alpha^q$ may be greater than 0. 
Hence, the maximum value of the step size of $e^*$ during the process is at most $\alpha^{\lceil \log_{\alpha} (Y^{\dagger}(e^*) - Y^{\ddagger}(e^*)) \rceil +1}$. 
Using the Additive Drift Theorem~\cite{he2004study} and the drift on the exponent of the step size of $e^*$ obtained above, the process that decreases the LP value of $e^*$ from $Y^{\dagger}(e^*)$ to $Y^{\ddagger}(e^*)$ needs at most 
$\Or \left(\log_{\alpha} \left(Y^{\dagger}(e^*) - Y^{\ddagger}(e^*) \! \right) - q \right) = \Or \left(\log_{\alpha} \left(Y^{\dagger}(e^*) - Y^{\ddagger}(e^*) \! \right) \! \right)$ mutations on $e^*$. That is, the process takes expected runtime $\Or \left(m \log_{\alpha} \left(Y^{\dagger}(e^*) - Y^{\ddagger}(e^*) \! \right) \! \right)$. 
\end{proof}

The following theorem can be derived based on the conclusions of Lemma~\ref{lem:RLS with 1/5-th Rule increase weight}. 

\begin{theorem}
\label{theo:RLS with 1/5-th Rule for DWVC-E+}
The expected runtime of the \mbox{\textup{RLS with 1/5-th Rule}} for \mbox{\textup{DWVC-E}}$^+$ is 
$\Or \big(\alpha m D \log_{\alpha} W_{\textup{max}} \cdot \log W_{\textup{max}} \! \big)$.
\end{theorem}
\begin{proof}
We study the expected runtime of the \mbox{\textup{RLS with 1/5-th Rule}} to obtain an MFDS $Y^*$ of $G^*=(V,E \cup E^+,W)$ starting with the MFDS $Y_{\textup{orig}}$ of $G = (V,E,W)$.
Observe that only the LP values of the edges in $E^+$ may have the room to be increased.
Moreover, Lemma~\ref{lem:RLS with 1/5-th Rule increase weight} gives that for each edge $e \in E^+$ with $Y^*(e) > Y_{\textup{orig}}(e)$, the algorithm takes expected runtime $\Or \left(\alpha m \log_{\alpha} W_{\textup{max}} \cdot \log W_{\textup{max}} \! \right)$ to increase its LP value from $Y_{\textup{orig}}(e)$ to $Y^*(e)$.
Therefore, combining it with the fact that $|E^+| = D$ directly gives the expected runtime 
$\Or \left(\alpha m D \log_{\alpha} W_{\textup{max}} \cdot \log W_{\textup{max}} \! \right)$.
\end{proof}

\begin{theorem}
\label{theo:RLS with 1/5-th Rule for DWVC-E-}
The expected runtime of the \mbox{\textup{RLS with 1/5-th Rule}} for \mbox{\textup{DWVC-E}}$^-$ is 
$\Or \big(\alpha m \log_{\alpha} W_{\textup{max}} \cdot \min\{m \log W_{\textup{max}}, D \cdot W_{\textup{max}} \}  \! \big)$.
\end{theorem}
\begin{proof}
We study the expected runtime of the \mbox{\textup{RLS with 1/5-th Rule}} to obtain an MFDS $Y^*$ of $G^*=(V,E \setminus E^-,W)$ starting with the given MFDS $Y_{\textup{orig}}$ of $G = (V,E,W)$.
Let $E_{\Delta}$ be the edges $e \in E^* = E \setminus E^-$ with $Y^*(e) > Y_{\textup{orig}}(e)$, and let $\Delta(e) = Y^*(e) - Y_{\textup{orig}}(e)$ for each edge $e \in E_{\Delta}$.
The reasoning given in Theorem~\ref{theo:(1+1) EA for DWVC-E-} shows that $\sum_{e \in E_{\Delta}} \Delta(e) \le D \cdot W_{\textup{max}}$.

Lemma~\ref{lem:RLS with 1/5-th Rule increase weight} gives that for each edge $e \in E_{\Delta}$, the algorithm takes expected runtime $\Or \left(\alpha m \log_{\alpha} W_{\textup{max}} \cdot \log \Delta(e) \! \right)$ to increase the LP value of $e$ from $Y_{\textup{orig}}(e)$ to $Y^*(e)$.
Summing the expected runtime over all the edges in $E_{\Delta}$ gives the expected runtime 
$\Or \left(\alpha m \log_{\alpha} W_{\textup{max}} \cdot \sum_{e \in E_{\Delta}}\log \Delta(e) \! \right)$.
For the upper bound of $\sum_{e \in E_{\Delta}}\log \Delta(e)$, we have
\begin{eqnarray}
\sum_{e \in E_{\Delta}} \log \Delta(e) &=& \log \left(\prod_{e \in E_{\Delta}} \Delta(e) \right) \\
&\le& |E_{\Delta}| \cdot \log \frac{\sum_{e \in E_{\Delta}} \Delta(e)}{|E_{\Delta}|} \\ 
&\le& |E_{\Delta}| \cdot \log \frac{D \cdot W_{\textup{max}}}{|E_{\Delta}|} \\ 
\label{ee13}
&\le& \frac{D \cdot W_{\textup{max}}}{e} \cdot \log e \ .
\end{eqnarray}
Inequality~\ref{ee13} can be obtained by the observation that $f(x) = x \cdot \log (D \cdot W_{\textup{max}}/x)$ ($x > 0$) gets its maximum value when $x = D \cdot W_{\textup{max}}/e$. 
Thus the expected runtime of the \mbox{\textup{RLS with 1/5-th Rule}} to obtain an MFDS $Y^*$ of $G^*$ starting with $Y_{\textup{orig}}$ can be bounded by $\Or \left(\alpha m D \log_{\alpha} W_{\textup{max}} \cdot W_{\textup{max}} \! \right)$.

Additionally, as $|E_{\Delta}| \le m$, and $\Delta(e) \le W_{\textup{max}}$ for each edge $e \in E_{\Delta}$, we have that the expected runtime of the algorithm can also be bounded by $\Or \left(\alpha m^2 \log_{\alpha} W_{\textup{max}} \cdot \log W_{\textup{max}} \! \right)$. Therefore, combining the two expected runtime given above, we have the claimed result.
\end{proof}


The following theorem can be derived using the conclusions obtained by Lemmata~\ref{lem:RLS with 1/5-th Rule increase weight} and~\ref{lem:RLS with 1/5-th Rule decrease weight}, and the reasoning similar to that given in Theorems~\ref{theo:RLS with 1/5-th Rule for DWVC-E-},~\ref{theo:(1+1) EA for DWVC-E-}, and~\ref{theo:(1+1) EA for DWVC-W-}.

\begin{theorem}
\label{theo:RLS with 1/5-th Rule for DWVC}
The expected runtime of the \mbox{\textup{RLS with 1/5-th Rule}} for \mbox{\textup{DWVC-X}}, where \textup{X} $\in \{E,W^+,W^-,W\}$, is $\Or \big(\alpha m \log_{\alpha} W_{\textup{max}} \cdot \min\{m \log W_{\textup{max}}, D \cdot W_{\textup{max}} \}  \! \big)$.
\end{theorem}






In the following, we analyze the performance of the \mbox{\textup{(1+1) EA with 1/5-th Rule}} for DWVC.
Firstly, we give a specific graph $G_s$ (see Figure~\ref{fig:special graph}(a)) that is the same as the graph considered in~\cite{pourhassan2017use}. W.l.o.g., we assume that the maximum weight $W_{\textup{max}}$ that the vertices in $G_s$ have is $\alpha^m$. 
Then we show that in a special situation, the \mbox{\textup{(1+1) EA with 1/5-th Rule}} requires pseudo-polynomial runtime to obtain the unique MFDS $Y^*$ of $G_s$, where $Y^*(e_1) = W_{\textup{max}}$ and $Y^*(e_i) = 1$ for all $2 \leq i \leq m$. 
The Chernoff-Hoeffding Bound given below is used in the proof for the main result stated later.

\medskip

\noindent{\bf Chernoff-Hoeffding Bound}~\cite{phillips2012chernoff}. Let $x_1,\dots,x_n$ be independent random variables such that $a_i \leq x_i \leq b_i$ for all $1 \leq i \leq n$. Denote $X = \sum_{i=1}^{n} x_i$. Then for any $\delta \geq 0$, the following inequality holds.
$${\rm Prob}(X \geq E[X] + \delta) \leq e^{-2 \delta^2 / \sum_{i=1}^{n} (b_i - a_i)^2} \ .$$

\medskip

\begin{figure}
\begin{center}
\vspace*{.25cm}
\begin{picture}(300,70)

\put(-20,10){\begin{picture}(0,0)
 \put(2,50){$W_{\textup{max}}$}
 \put(2,-5){$W_{\textup{max}}$}
 \put(-1,22){$e_1$}
 \put(10,5){\circle*{3}}
 \put(10,5){\line(0,1){40}}
 \put(10,45){\circle*{3}}

 \put(32,49){$1$}
 \put(32,-6){$1$}
 \put(24,22){$e_2$}
 \put(35,5){\circle*{3}}
 \put(35,5){\line(0,1){40}}
 \put(35,45){\circle*{3}}

 \put(57,49){$1$}
 \put(57,-6){$1$}
  \put(49,22){$e_3$}
 \put(60,5){\circle*{3}}
 \put(60,5){\line(0,1){40}}
 \put(60,45){\circle*{3}}

  \put(80,22){$\ldots$}

 \put(107,49){$1$}
 \put(107,-6){$1$}
 \put(98,22){$e_m$}
 \put(110,5){\circle*{3}}
 \put(110,5){\line(0,1){40}}
 \put(110,45){\circle*{3}}
 \end{picture}}

 \put(35,-7){\small (a)}

 \put(210,10){\begin{picture}(0,0)
 \put(-18,50){$W_{\textup{max}}$}
 \put(2,-5){$W_{\textup{max}}$}
\put(-38,-5){$W_{\textup{max}}$}
 \put(4,22){$e_1$}
 \put(-32,22){$e'_1$}
 \put(10,5){\circle*{3}}
 \put(-30,5){\circle*{3}}
 
 \put(-10,45){\circle*{3}}

 \put(-10,45){\line(-1,-2){20}}
\put(-10,45){\line(1,-2){20}}

 \put(32,49){$1$}
 \put(32,-6){$1$}
 \put(24,22){$e_2$}
 \put(35,5){\circle*{3}}
 \put(35,5){\line(0,1){40}}
 \put(35,45){\circle*{3}}

 \put(57,49){$1$}
 \put(57,-6){$1$}
  \put(49,22){$e_3$}
 \put(60,5){\circle*{3}}
 \put(60,5){\line(0,1){40}}
 \put(60,45){\circle*{3}}

  \put(80,22){$\ldots$}

 \put(107,49){$1$}
 \put(107,-6){$1$}
 \put(98,22){$e_m$}
 \put(110,5){\circle*{3}}
 \put(110,5){\line(0,1){40}}
 \put(110,45){\circle*{3}}
 \end{picture}}
\put(265,-7){\small (b)}

\end{picture}
\end{center}
\vspace*{-1mm}
\caption{(a). The special graph $G_s$ contains $m$ edges, each of which is independent (i.e., each edge constitutes a connected component of $G_s$). Except the two endpoints of edge $e_1$ in $G_s$ that have weight $W_{\textup{max}}$, all the other vertices have weight 1. (b). The graph $G'_s$ is a variant of $G_s$ with an additional vertex and an additional edge $e'_1$.}
\label{fig:special graph}
\vspace*{-3mm}
\end{figure}
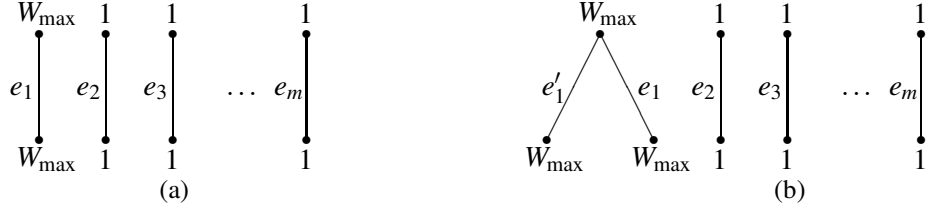

\begin{lemma}
\label{lem:(1+1) EA with 1/5-th Rule increase weight}
Consider a feasible dual-solution $Y^{\dagger}$ of $G_s$ with $Y^{\dagger}(e_i) = 1$ for all $1 \leq i \leq m$, and the step size of the edge $e_1$ with an initial value 1. 
The expected runtime of the \mbox{\textup{(1+1) EA with 1/5-th Rule}} to obtain the unique MFDS $Y^*$ of $G_s$ starting with $Y^{\dagger}$ is lower bounded by $2^{m^{\epsilon /2}}$ ($0 < \epsilon \leq 1/2$) with probability $1- e^{-{\rm \Omega}(m^{\epsilon})}$.
\end{lemma}

\begin{proof}  
Let $M$ be a mutation of the \mbox{\textup{(1+1) EA with 1/5-th Rule}} that selects the edge $e_1$ (note that $M$ may also select some edges in addition to $e_1$).
If $M$ is accepted, then the exponent $q$ of the step size $\sigma(e_1) = \alpha^q$ ($q \ge 0$) of $e_1$ is increased by $1$; otherwise, decreased by $1/4$. 
Observe that $M$ can be accepted only if $M$ just selects the edge $e_1$ and increases its LP value. 
Thus the probability $P_{inc}$ that the mutation $M$ is accepted is $\leq \frac{1}{2e}$, and the probability $P_{dec}$ that the mutation $M$ is rejected is $\ge 1 - \frac{1}{2e}$. 
As the drift of the exponent $q$ is $1 \cdot P_{inc} + (-1/4) \cdot P_{dec} \leq (5-2e)/8e < 0$, $q$ will gradually decrease to 0 if $q > 0$, i.e., the step size of $e_1$ will gradually decrease to 1 if it is greater than 1. 

The step size of $e_1$ has an initial value 1, hence if it cannot increase to an enough large value during the whole process of $2^{m^{\epsilon /2}}$ steps with a high probability, then we can show that the LP value of $e_1$ after the $2^{m^{\epsilon /2}}$ steps cannot reach $W_{\textup{max}} = \alpha^{m}$ with a high probability. 
In the following discussion, we assume that the step size of $e_1$ is increased to $\alpha$ at some point, i.e., $q = 1$. 
Then we use Chernoff-Hoeffding Bound to show that $T_1$ is upper bounded by $m^{\epsilon}$ with probability $1 - e^{-{\rm \Omega}(m^{\epsilon})}$, where $T_1$ denotes the number of steps that the algorithm requires to decrease the step size of $e_1$ from $\alpha$ to 1. 
Observe that $x_i$ denotes the increment on the exponent $q$ of the step size of $e_1$, which equals $1$ or $-1/4$, $a_i = -1/4$, and $b_i = 1$ for all $1 \leq i \leq n$, where $n = m^{\epsilon}$ (the notations $x_i$, $b_i$, and $a_i$ follow the ones given in the definition of Chernoff-Hoeffding Bound). 
As the exponent $q$ has values 1 and 0 before and after $T_1$ steps, respectively, we have $X = \sum_{i = 1}^{n} x_i = -1$.
Furthermore, considering the equality $X = E[X] + \delta$, where $E[X] \leq \frac{5-2e}{8e} \cdot m^{\epsilon}$, Chernoff-Hoeffding Bound gives that the probability that $T_1 > m^{\epsilon}$ is upper bounded by 
$$e^{-2 \delta^2 / \sum_{i=1}^{n} (b_i - a_i)^2} = e ^{-2 [\frac{2e-5}{8e} \cdot m^{ \epsilon} - 1]^2 / (\frac{25}{16} m^{\epsilon})} = e^{-{\rm \Omega}(m^{\epsilon})} \ .$$ 
Meanwhile, we also have that the maximum value of the step size of $e_1$ during the $T_1$ steps is upper bounded by $\alpha^{m^{\epsilon}}$, with probability $1- e^{-{\rm \Omega}(m^{\epsilon})}$.

Now we consider the whole process of $2^{m^{\epsilon /2}}$ steps. A phase of the whole process is {\it non-trivial} if it starts with a point where the step size of $e_1$ is increased to $\alpha$, ends with a point where the step size of $e_1$ is decreased to 1 for the first time (i.e., the step sizes of $e_1$ at all internal points of the phase are greater than 1). 
Thus the whole process consists of $N_1$ non-trivial phases and $N_2$ steps where the step size of $e_1$ is 1 (both $N_1$ and $N_2$ are nonnegative integers).
For a non-trivial phase $P$, by the analysis given above, the number of steps in $P$ is upper bounded by $m^{\epsilon}$ with probability $1 - e^{-{\rm \Omega}(m^{\epsilon})}$, and lower bounded by $5$ (one step increases the step size to $\alpha$, and four steps decrease the step size to 1). 
Thus the number $N_1$ of non-trivial phases is upper bounded by $2^{m^{\epsilon /2}} /5$. 
Combining the conclusion obtained above that the step size is increased to over $\alpha^{m^{\epsilon}}$ during each non-trivial phase with probability $e^{-{\rm \Omega}(m^{\epsilon})}$, for the $N_1$ non-trivial phases, we have that the step size of $e_1$ is increased to over $\alpha^{m^{\epsilon}}$ with probability 
$$ e^{-{\rm \Omega}(m^{\epsilon})} \cdot N_1 \leq e^{-{\rm \Omega}(m^{\epsilon})} \cdot 2^{m^{\epsilon /2}} /5  = e^{-{\rm \Omega}(m^{\epsilon})} \ .$$
That is, during the whole process of $2^{m^{\epsilon /2}}$ steps, the step size of $e_1$ is increased to over $\alpha^{m^{\epsilon}}$ with probability $e^{-{\rm \Omega}(m^{\epsilon})}$. 
Therefore, by the end of the whole process of $2^{m^{\epsilon/2}}$ steps, the increment on the LP value of $e_1$ is upper bounded by $2^{m^{\epsilon /2}} \cdot \alpha^{m^{\epsilon}} \leq \alpha^{2 m^{\epsilon}}$ (as $\alpha \ge 2$) with probability $1 - e^{-{\rm \Omega}(m^{\epsilon})}$, where $\alpha^{2 m^{\epsilon}}$ is less than $W_{\textup{max}} -1 = \alpha^m -1$ since $0 < \epsilon \leq 1/2$ and $m$ is sufficiently large. 
Therefore, with probability $1- e^{-{\rm \Omega}(m^{\epsilon})}$, the \mbox{\textup{(1+1) EA with 1/5-th Rule}} cannot find the unique MFDS of $G^*$ within runtime $2^{m^{\epsilon /2}}$.
\end{proof}

There are two reasons for the pseudo-polynomial runtime of the \mbox{\textup{(1+1) EA with 1/5-th Rule}} for DWVC-E$^+$: (1). the small probability of a mutation to be accepted by the algorithm; (2). the "radical" strategy that decreases the step sizes of all the edges selected by the mutation if it is rejected. 
Under the combined impact of the two factors, the step size of $e_1$ cannot be increased to an enough large value to overcome the exponential large weight $W_{\textup{max}}$.
An obvious workaround is incorporating the ``conservative'' strategy (adopted by Algorithm~\ref{alg:(1+1) EA}) into the \mbox{\textup{(1+1) EA with 1/5-th Rule}}, which only decreases the step sizes of the edges that satisfy a strict condition. Then the probability of a mutation that decreases the step size of $e_1$ would be smaller.
Another possible workaround is considering the $1/i$-th rule, where $i > 5$, to slow down the decreasing speed of the step size of $e_1$. 
Both workarounds aim to make the expected drift of the step size of $e_1$ be positive.

Considering the instance $\{G_s \setminus \{e_1\},Y_{\textup{orig}},E^+ = \{e_1\} \! \}$ of DWVC-E$^+$, where $Y_{\textup{orig}}(e_i) = 1$ for all $2 \leq i \leq m$, we can get that Theorem~\ref{theo:(1+1) EA with 1/5-th Rule for DWVC-X} holds for \mbox{\textup{DWVC-E}}$^+$ by Lemma~\ref{lem:(1+1) EA with 1/5-th Rule increase weight}.
Similarly, considering the instance $\{G'_s,Y_{\textup{orig}},E^- = \{e'_1\} \! \}$ of DWVC-E$^-$, where $Y_{\textup{orig}}(e_i) = 1$ for all $1 \leq i \leq m$ and $Y_{\textup{orig}}(e'_1) = W_{\textup{max}}-1$ (graph $G'_s$ is given in Figure~\ref{fig:special graph}(b)), we can get that Theorem~\ref{theo:(1+1) EA with 1/5-th Rule for DWVC-X} holds for \mbox{\textup{DWVC-E}}$^-$ by Lemma~\ref{lem:(1+1) EA with 1/5-th Rule increase weight}.
Considering that the weights of the two endpoints of $e_1$ in $G_s$ are increased from 1 to $W_{\textup{max}}$, and $Y_{\textup{orig}}(e_i) = 1$ for all $1 \leq i \leq m$, we can get that Theorem~\ref{theo:(1+1) EA with 1/5-th Rule for DWVC-X} holds for \mbox{\textup{DWVC-W}}$^+$ by Lemma~\ref{lem:(1+1) EA with 1/5-th Rule increase weight}.
Considering that the weight of the endpoint of $e'_1$ that is not shared with $e_1$ in $G'_s$ is decreased from $W_{\textup{max}}$ to 1, and $Y_{\textup{orig}}(e'_1) = W_{\textup{max}}-1$ and $Y_{\textup{orig}}(e_i) = 1$ for all $1 \leq i \leq m$, we can get that Theorem~\ref{theo:(1+1) EA with 1/5-th Rule for DWVC-X} holds for \mbox{\textup{DWVC-W}}$^-$ by Lemma~\ref{lem:(1+1) EA with 1/5-th Rule increase weight}.
Combining the conclusions for \mbox{\textup{DWVC-E}}$^+$, \mbox{\textup{DWVC-E}}$^-$, \mbox{\textup{DWVC-W}}$^+$, and \mbox{\textup{DWVC-W}}$^-$, we have that Theorem~\ref{theo:(1+1) EA with 1/5-th Rule for DWVC-X} holds for \mbox{\textup{DWVC-E}} and \mbox{\textup{DWVC-W}}.

\begin{theorem}
\label{theo:(1+1) EA with 1/5-th Rule for DWVC-X}
If the maximum weight $W_{\textup{max}}$ of the vertices in the considered weighted graph is $\alpha^{m}$, then the expected runtime of the \mbox{\textup{(1+1) EA with 1/5-th Rule}} for \mbox{\textup{DWVC-X}}, where \textup{X} $\in \{E^+, E^-,E,W^+, W^-, W\}$, is lower bounded by $2^{m^{\epsilon /2}}$ with probability $1- e^{-{\rm \Omega}(m^{\epsilon})}$ ($0 < \epsilon \leq 1/2$).
\end{theorem}

\section{Conclusion}

In the paper, we contributed to the theoretical understanding of evolutionary computing for the Dynamic Weighted Vertex Cover problem, generalizing the results obtained by Pourhassan et al.~\cite{pourhassan2015maintaining} for the Dynamic Vertex Cover problem. Two graph-editing operations were studied for the dynamic changes on the given weighted graph, which lead to two versions: Dynamic Weighted Vertex Cover problem with Edge Modification and Dynamic Weighted Vertex Cover problem with Weight Modification, and two special variants for each version. 

We first introduced two algorithms (1+1) EA and RLS with the step size adaption strategy, and analyzed their performances for the two versions (including their four special variants) separately.
Our analysis shows that the qualities of the solutions for these studied dynamic changes can be maintained efficiently.
As mentioned in Section~\ref{sec:intro}, Pourhassan et al.~\cite{pourhassan2017use} studied the Weighted Vertex Cover problem using the dual form of the LP formulation, and showed that their (1+1) EA with step size adaption strategy cannot get a 2-approximate solution in polynomial expected runtime with a high probability. 
It is easy to find that our (1+1) EA can be extended to solve the Weighted Vertex Cover problem efficiently (i.e., construct a 2-approximate solution), of which each instance $G' = (V',E',W')$ can be transformed to an instance of DWVC-E$^+$ with $E = \emptyset$ and $E^+ = E'$.
There are two main differences between their (1+1) EA and our (1+1) EA, causing the big performance gap: 
(1). for the mutation $M$ of their (1+1) EA, the adjustment directions of the LP values of the edges selected by $M$ are random, i.e., there may exist two edges selected by $M$ whose LP values are increased and decreased respectively; for the mutation $M$ of our (1+1) EA, the LP values of the edges selected by $M$ are either all increased or all decreased, and the adjustment direction depends on the feasibility of the maintained solution; 
(2). for the mutation $M$ that is rejected by their (1+1) EA, the step sizes of all the edges selected by $M$ are decreased; for the mutation $M$ that is rejected by our (1+1) EA, only the step sizes of the edges satisfying a specific condition can be decreased.

To eliminate the artificial influences on the behaviors of the two algorithms mentioned above, we also incorporated the 1/5-th (success) rule to control the increasing/decreasing rate of the step size, and presented two algorithms named \mbox{\textup{(1+1) EA with 1/5-th Rule}} and \mbox{\textup{RLS with 1/5-th Rule}}. 
The \mbox{\textup{RLS with 1/5-th Rule}} was shown to be able to maintain the qualities of the solutions efficiently as well. However, for the \mbox{\textup{(1+1) EA with 1/5-th Rule}}, its performance was shown to be not satisfying. More specifically, its expected runtime is lower bounded by a pseudo-polynomial with a high probability, to maintain the qualities of the solutions if the maximum weight that the vertices have is exponential with respect to the number of the edges in the graph. 
The result matches that given by Pourhassan et al.~\cite{pourhassan2017use}, and indicates that the 1/5-th rule cannot overcome the negative impact caused by the standard mutation operator, when considering the special instances. 
However, for the 1/$i$-th rule with a sufficiently large value of $i$, it seems to be a promising way to overcome such an impact. 

This is the first work that incorporates the 1/5-th rule with the step size adaption strategy, to solve a dynamic combinatorial optimization problem. We will leave it for future research to extend these insights to the analysis for more (dynamic) combinatorial optimization problems. 



\end{document}